\documentclass[letterpaper,11pt]{article}

\usepackage[utf8]{inputenc}
\usepackage[T1]{fontenc}
\usepackage{kpfonts}
\usepackage{microtype}

\usepackage[verbose=true,letterpaper,margin=1.25in]{geometry}
\usepackage{microtype}
\usepackage{graphicx}
\usepackage{url}            
\usepackage{booktabs}       
\usepackage{amsfonts}       
\usepackage{nicefrac}       
\usepackage{microtype}      
\usepackage{lipsum}
\usepackage{amsmath}
\usepackage{float}
\usepackage{array}
\usepackage{xspace}
\usepackage{amsthm}
\usepackage{multirow}
\usepackage[export]{adjustbox}
\usepackage{stfloats}
\usepackage{natbib}
\usepackage{xcolor}
\usepackage{color}
\usepackage{wrapfig}
\usepackage{subfig}
\usepackage{bbm}
\usepackage{caption}
\DeclareCaptionLabelFormat{andtable}{#1~#2  \&  \tablename~\thetable}

\usepackage{mystyle}

\usepackage{etoolbox}

\patchcmd{\abstract}{\small}{}{}{}

\definecolor{DarkGreen}{rgb}{0.1,0.5,0.1}
\definecolor{DarkRed}{rgb}{0.5,0.1,0.1}
\definecolor{DarkBlue}{rgb}{0.1,0.1,0.5}
\definecolor{Gray}{rgb}{0.2,0.2,0.2}

\definecolor{negvals}{rgb}{0.1,0.5,0.1}
\definecolor{posvals}{rgb}{0.5,0.1,0.1}
\definecolor{neutralvals}{RGB}{255, 140, 0}
\definecolor{neutralvals}{rgb}{0.6, 0.51, 0.48}
\definecolor{posvals}{rgb}{0.8, 0.0, 0.0}

\usepackage{hyperref}       
\hypersetup{
    unicode=false,          
    pdftoolbar=true,        
    pdfmenubar=true,        
    pdffitwindow=false,      
    pdfnewwindow=true,      
    colorlinks=true,       
    linkcolor=DarkBlue,          
    citecolor=DarkGreen,        
    filecolor=DarkRed,      
    urlcolor=DarkBlue,          
}

\usepackage{booktabs}

\newtheorem{theorem}{Theorem}[section]

\newtheorem{proposition}[theorem]{Proposition}

\theoremstyle{definition}

\newcommand{\XCal}{\mathcal{X}}

\newcommand{\DCal}{\mathcal{D}}

\newcommand{\YCal}{\mathcal{Y}}

\usepackage{array}
\newcommand{\PreserveBackslash}[1]{\let\temp=\\#1\let\\=\temp}
\newcolumntype{C}[1]{>{\PreserveBackslash\centering}p{#1}}
\newcolumntype{R}[1]{>{\PreserveBackslash\raggedleft}p{#1}}
\newcolumntype{L}[1]{>{\PreserveBackslash\raggedright}p{#1}}

\usepackage{multirow}

\usepackage[font=normalsize,labelfont=bf, labelsep=colon, textfont=it]{caption}

\usepackage{algorithm,algorithmic}
\usepackage[algo2e]{algorithm2e}

\DeclareMathOperator*{\argmin}{arg\,min}

\DeclareMathOperator*{\E}{\mathbb{E}}

\usepackage[affil-it]{authblk} 
\title{Test-time collective prediction}
\author{
    Celestine Mendler-D\"{u}nner \quad 
    Wenshuo Guo \quad Stephen Bates \quad Michael I.\,Jordan
{\small \{{mendler,\,wguo,\,stephenbates,\,jordan}\}@berkeley.edu}\\ 
 University of California, Berkeley}
\date{}

\usepackage{hyperref}

\begin{document}

\vspace{-0.5cm}

\maketitle

\begin{abstract}
  An increasingly common setting in machine learning involves multiple parties, each with their own data, who want to jointly make predictions on future test points. Agents wish to benefit from the collective expertise of the full set of agents to make better predictions than they would individually, but may not be willing to release their data or model parameters. 
In this work, we explore a decentralized mechanism to make collective predictions at test time, leveraging each agent’s pre-trained model without relying on external validation, model retraining, or data pooling. Our approach takes inspiration from the literature in social science on human consensus-making. We analyze our mechanism theoretically, showing that it converges to inverse mean-squared-error (MSE) weighting in the large-sample limit. To compute error bars on the collective predictions we propose a decentralized Jackknife procedure that evaluates the sensitivity of our mechanism to a single agent's prediction. Empirically, we demonstrate that our scheme effectively combines models with differing quality across the input space. The proposed consensus prediction achieves significant gains over classical model averaging, and even outperforms weighted averaging schemes that have access to additional validation data.

\end{abstract}

\vspace{0.5cm}

\vspace{-0.5cm}
\section{Introduction}

Large-scale datasets are often collected from diverse sources, by multiple parties, and stored across different machines. In many scenarios centralized pooling of data is not possible due to privacy concerns, data ownership, or storage constraints. The challenge of doing machine learning in such distributed and decentralized settings has motivated research in areas of federated learning~\citep{federated15,fed17macmahan}, distributed learning~\citep{distNN12,splitlearning18}, as well as hardware and software design~\citep{spark16,ray18moritz}. 

While the predominant paradigm in distributed machine learning is \emph{collaborative learning} of one centralized model, this level of coordination across machines at the training stage is sometimes not feasible. In this work, we instead aim for \emph{collective prediction} at test time without posing any specific requirement at the training stage. Combining the predictions of pre-trained machine learning models has been considered in statistics and machine learning in the context of ensemble methods, including bagging~~\citep{bagging96breiman} and stacking~\citep{stacking92wolpert}.
In this work, we explore a new perspective on this aggregation problem and investigate whether insights from the social sciences on how humans reach a consensus can help us design more effective aggregation schemes that fully take advantage of each model's individual strength. At a high level, when a panel of human experts come together to make collective decisions, they exchange information and share expertise through a discourse to then weigh their subjective beliefs accordingly. This gives rise to a dynamic process of consensus finding that is decentralized and does not rely on any external judgement. We map this paradigm to a machine learning setting where experts correspond to pre-trained machine learning models, and show that it leads to an appealing mechanism for test-time collective prediction that does not require any data sharing, communication of model parameters, or labeled validation data.
\begin{figure}[t]
    \centering
    \includegraphics[width=0.75\textwidth]{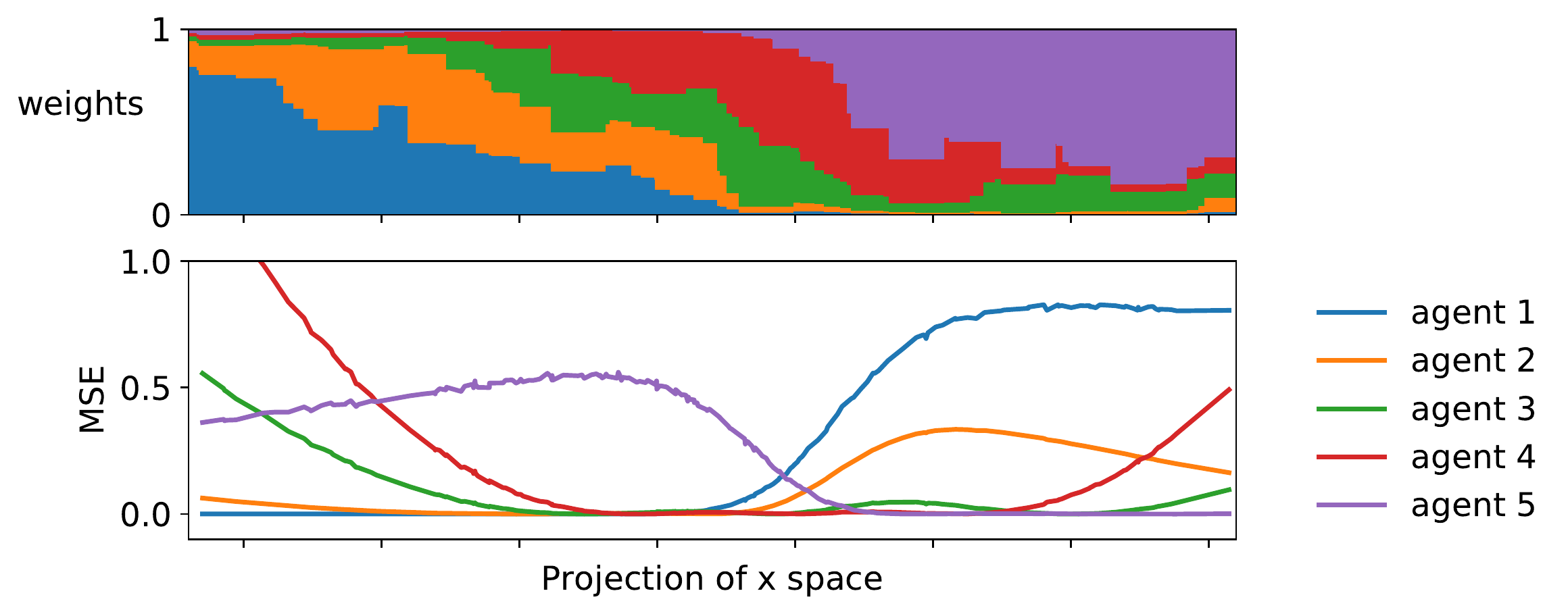} 
    \caption{\emph{(Test-time aggregation of heterogeneous models).} The proposed method upweights more accurate models (lower MSE) adaptively in different regions of the input space. Constructed from a regression task described in Section~\ref{sec:synthetic}.\vspace{-0.2cm}}
    \label{fig:teaser}
\end{figure}

\subsection{Our work}

We build on a model for human consensus finding that defines an iterative process of opinion pooling based on \emph{mutual trust scores} among agents~\citep{degroot1974consensus}. Informally speaking, a mutual trust score reflects an agent's willingness to adapt another agent's opinion. Thus, each individual's impact on the final judgment 
depends on how much it is trusted by the other agents.

Mapping the DeGroot model to our learning scenario, we develop a scheme where each agent uses its own local training data to assess the predictive accuracy of other agents' models and to determine how much they should be trusted. 
This assessment of mutual trust is designed to be adaptive to the prediction task at hand, by using only a subset of the local data in the area around the current test point.
Then, building on these trust scores, a final judgment is reached through an iterative procedure, where agents repeatedly share their updated predictions. 

This aggregation scheme is conceptually appealing as it mimics human consensus finding, and exhibits several practical advantages over existing methods. 
First, it is decentralized and does not require agents to release their data or model parameters. This can be beneficial in terms of  communication costs and attractive from a privacy perspective, as it allows agents to keep their data local and control access to their model. 
Second, it does not require hold-out validation data to construct or train the aggregation function. Thus, there is no need to collect additional labeled data and coordinate shared access to a common database.
Third, it does not require any synchronization across agents during the training stage, so agents are free to choose their preferred model architecture, and train their models independently. 

A crucial algorithmic feature of our procedure is its \emph{adaptivity} to the test point at hand, which allows it to deal with inherent variations in the predictive quality across models in different regions of the input space. 
Figure~\ref{fig:teaser} illustrates how our mechanism adaptively upweights models with lower error depending on the test point's location in the input space. In fact, we prove theoretically that our aggregation procedure can recover inverse-MSE weighting in the large-sample limit, which is known to be optimal for variance reduction of unbiased estimators~\citep{bates69}.

To assess our procedure's sensitivity to the predictions of individual agents we propose a decentralized Jackknife algorithm to compute error bars for the consensus predictions with respect to the collection of observed agents. These error bars offer an attractive target of inference since they are agnostic to how individual models have been trained, can be evaluated without additional model evaluations, and allow one to diagnose the vulnerability of the algorithm to malicious agents.

On the empirical side, we demonstrate the efficacy of our mechanism through extensive numerical experiments across different learning scenarios. In particular, we illustrate the mechanism's advantages over model averaging as well as model selection, and demonstrate that it consistently outperforms alternative non-uniform combination schemes that have access to additional validation data across a wide variety of models and datasets.

\subsection{Related work}

The most common approach for combining expert judgment is to aggregate and average individual expert opinions~\citep{hammitt13}. 
In machine learning this statistical approach of equally-weighted averaging is also known as the ``mean-rule'' and underlies the popular \emph{bagging} technique~\citep{bagging96breiman, buja2006bagging, efron2014estimation}. It corresponds to an optimal aggregation technique for dealing with uncertainty in settings where individual learners are homogeneous and the local datasets are drawn independently from the same underlying distribution. While principled, and powerful due to its simplicity, an unweighted approach cannot account for heterogeneity in the quality, or the expertise, of the base learners. 

To address this issue, performance-weighted averaging schemes have been proposed in various contexts. For determining an appropriate set of weights, these approaches typically rely on calibrating variables~\citep{aspinall10} to detect differences in predictive quality among agents. In machine learning these calibrating variables correspond to labeled cross-validation samples, while in areas such as expert forecasting and risk analysis these variables take the form of a set of agreed-upon seed questions~\citep{MCooke1991uncer}. 
These weights can be updated over time using statistical information about the relative predictive performance~\citep{Rogova2008} or the posterior model probability~\citep{wasserman2000bayesian}.
More recently there have been several applications to classification that focus on dynamic selection techniques \citep{dyna18cruz}, where the goal is to select the locally best classifier from a pool of classifiers on the fly for every test point. At their core, these approaches all rely on constructing a good proxy for local accuracy in a region of the feature space around the test point and to use this proxy to weight or rank the candidate models. Various versions of distance and accuracy metrics have been proposed in this context~\citep[eg.][]{woods97,didaci05,cruz18}.
However, all these methods rely crucially on shared access to labeled validation data to determine the weights.
Similarly, aggregation rules such as \emph{stacking}~\citep{stacking96breiman,NNstacking20} or \emph{meta-models}~\citep{Albardan2020SPOCCSP}, both of which have proved their value empirically, require a central trusted agency with additional data to train an aggregation function on top of the predictor outputs. A related model-based aggregation method, the mixture-of-experts model~\citep{MEM}, requires joint training of all models as well as data pooling. 

To the best of our knowledge there is no approach in the literature that can perform adaptive performance weighting without relying on a central agent to validate the individual models. Notably, our approach differs from general frameworks such as model parameter mixing~\citep{avgmix13zhang}, model fusion~\citep{fusion18, fusion20singh} in that agents are not required to release their models.

Outside machine learning, the question of how to combine opinions from different experts to arrive at a better overall outcome has been studied for decades in risk analysis~\citep{hanes18riskweights,robert99} and management science~\citep{combiningexperts77morris}. 
Simplifying the complex mechanism of opinion dynamics, the celebrated DeGroot model~\citep{degroot1974consensus} offers an appealing abstraction to reason mathematically about human consensus finding. It has been popular in the analysis of convergent beliefs in social networks and it can be regarded as a formalization of the Delphi technique, which has been used to forecast group judgments related to the economy, education, healthcare, and public policy~\citep{Dalkey1972STUDIESIT}. 
Extensions of the DeGroot model have been proposed to incorporate more sophisticated socio-psychological features that
may be involved when individuals interact, such as homophily~\citep{Hegselmann02opiniondynamics}, competitive interactions~\citep{altafini02antagonistic} and stubbornness~\citep{Friedkin11380}.

\section{Preliminaries}

We are given $K$ independently trained models and aim to aggregate their predictions at test time. Specifically, we assume there are $K$ agents, each of which holds a local training dataset, $\DCal_k$, of size $n_k$, and a model $f_k:\XCal \rightarrow \YCal$. We do not make any prior assumption about the quality or structure of the individual models, other than that they aim to solve the same regression tasks. 

At test time, let $x'\in\XCal$ denote the data point that we would like to predict. Each agent $k\in[K]$ has her own prediction for the test point corresponding to $ f_k(x')$. Our goal is to design a mechanism $\mathcal M$ that combines these predictions into an accurate single prediction, $p^\ast(x')=\mathcal M (f_1(x'),...,f_K(x'))$. We focus on the squared loss as a measure of predictive quality.

The requirements we pose on the aggregation mechanism $\mathcal M$ are: (i) agents should not be required to share their data or model parameters with other agents, or any external party; (ii) the procedure should not assume access to additional labeled data for an independent performance assessment or training of a meta-model.

\subsection{The DeGroot consensus model}
\label{sec:DG}

The consensus model proposed by~\citet{degroot1974consensus} specifies how agents in a group influence each other's opinion on the basis of mutual trust. 
Formally, we denote the trust of agent $i$ towards the belief of agent $j$ as $\tau_{ij} \in (0,1]$, with $\sum_{j \in [K]} \tau_{ij} = 1$ for every $i$. Then, DeGroot defines an iterative procedure whereby each agent $i$ repeatedly updates her belief $p_i$ as
\begin{equation}
\label{eq:consensus}
    p_i^{(t+1)}= \sum_{j=1}^K \tau_{ij} \, p_j^{(t)}.
\end{equation}
The procedure is initialized with initial beliefs $p_i^{(0)}$ and run until $p_i^{(t)}=p_j^{(t)}$ for every pair $i,j\in[K]$, at which point a consensus is deemed to have been reached. We denote the consensus belief by $p^\ast$. 

The DeGroot model has a family resemblance to the influential PageRank algorithm~\citep{page1998pagerank} that is designed to aggregate information across webpages in a decentralized manner. Similar to the consensus in DeGroot, the page rank corresponds to a steady state of equation \eqref{eq:consensus}, where the weights $\tau_{ij}$ are computed based on the fraction of outgoing links pointing from page $i$ to page $j$.

\subsection{Computing a consensus}

The DeGroot model can equivalently be thought of as specifying a weight for each agent's prediction in a linear \emph{opinion pool}~\citep{opinionpool61stone}. This suggests another way to compute the consensus prediction.
Let us represent the mutual trust scores in the form of a stochastic matrix, $T = \{\tau_{ij}\}_{i,j=1}^K$, which defines the state transition probability matrix of a Markov chain with $K$ states. The stationary state $w$ of this Markov chain, satisfying $wT=w$, defines the relative weight of each agent's initial belief in the final consensus 
\begin{equation}
    p^\ast = \sum_{j=1}^K w_j \, p_j^{(0)}.
\end{equation} The Markov chain is guaranteed to be irreducible and aperiodic since all entries of $T$ are positive. Therefore, a unique stationary distribution $w$ exists, and consensus will be reached from any initial state~\citep{degrootconvergence77}. We refer to the weights $w$ as \emph{consensus weights}.

To compute the consensus prediction $p^\ast$ one can solve the eigenvalue problem numerically for $w$ via power iteration. Letting $T^\ast=\lim_{m\rightarrow \infty}T^m $, the rows of $T^\ast$ identify, and correspond to the weights $w$ that DeGroot assigns to the individual agents~\citep[Theorem~5.6.6]{durrett2019probability}. The consensus is computed as a weighted average over the initial beliefs.

\section{DeGroot aggregation for collective test-time prediction}

In this section we discuss the details of how we implement the DeGroot model outlined in Section~\ref{sec:DG} for defining a mechanism $\mathcal M$ to aggregate predictions across machine learning agents.
We then analyze the resulting algorithm theoretically, proving its optimality in the large-sample limit, and we outline how the sensitivity of the procedure with respect to individual agents can be evaluated in a decentralized fashion using a jackknife method. 

\subsection{Local cross-validation and mutual trust score} 
A key component of DeGroot's model for consensus finding is the notion of mutual trust between agents, built by information exchange in form of a discourse. Given our machine learning setup with $K$ agents each of which holds a local dataset and a model $f_k$, we simulate this discourse by allowing query access to predictions from other models. This enables each agent to validate another agent's predictive performance on its own local data. This operationalizes a proxy for trustworthiness.

An important requirement for our approach is that it should account for model heterogeneity, and define trust in light of the test sample $x'$ we want to predict. Therefore, we design an \emph{adaptive} method, where the mutual trust $\tau_{ij}$ is reevaluated for every query via a local cross-validation procedure. Namely, agent $i$ evaluates all other agents' predictive accuracy using a subset $\DCal_i(x')\subseteq \DCal_i$ of the local data points that are most similar to the test point $x'$~\citep[cf.][]{woods97}. More precisely, agent $i$ evaluates
\begin{equation}
    \text{MSE}_{ij}(x')= \frac{1}{|\DCal_i(x')|} \sum_{(x,y) \in \DCal_i(x')}({f}_j(x)- y)^2,
    \label{eq:localacc}
\end{equation}
locally for every agent $j\in[K]$, and then performs normalization to obtain the trust scores:
\begin{equation}
    \tau_{ij}=\frac{1/\text{MSE}_{ij}}{\sum_{j\in[K]} 1/\text{MSE}_{ij}}.
    \label{eq:taunormalize}
 \end{equation}

There are various ways one can define the subset $\DCal_i(x')$; see, e.g.,~\citet{dyna18cruz} for a discussion of relevant distance metrics in the context of dynamic classifier selection. 
Since we focus on tabular data in our experiments we use Euclidean distance and assume $\DCal_i(x')$ has fixed size. Other distance metrics, or alternatively a kernel-based approach, could readily be accommodated within the DeGroot aggregation procedure.

\begin{algorithm}[!t]
\caption{DeGroot Aggregation}\label{alg:DG}
\begin{algorithmic}[1]
\STATE \textbf{Input:} $K$ agents with pre-trained models $f_1, \cdots, f_K$ and local data $\DCal_k, k \in [K]$; \\ neighborhood size $N$; test point $x'$.\\[1ex]
\STATE \underline{construct trust scores} (based on local accuracy):\\[1ex]
\FOR{$i=1,2,\ldots, K$}
\STATE Construct local validation dataset $\DCal_i(x')$ using $N$-nearest neighbors of $x'$ in $\DCal_i$.
\STATE Compute accuracy $\text{MSE}_{ij}(x')$ of other agents $j=1\ldots,K$ on $\DCal_i(x')$ according to \eqref{eq:localacc}.
\STATE Evaluate local trust scores $\{\tau_{ij}\}_{j\in[K]}$ by normalizing $\text{MSE}_{ij}$ as in \eqref{eq:taunormalize}.
\ENDFOR\\[1ex]
\STATE \underline{find consensus:}\\
\STATE Run pooling iterations \eqref{eq:consensus}, initialized at $p_j^{(0)}=f_j(x')$  $\forall j$, until a consensus $p^\ast$ is reached.\\
\STATE \textbf{Return:} Collective prediction $p^*(x')=p^\ast$\\[1ex]
\end{algorithmic}
\end{algorithm}

\subsection{Algorithm procedure and basic properties}

We now describe our overall procedure. We take the trust scores from the previous section, and use them to aggregate the agents' predictions via the DeGroot consensus mechanism. See Algorithm~\ref{alg:DG} for a full description of the DeGroot aggregation procedure. In words, after each agent evaluates her trust in other agents, she repeatedly  pools their updated predictions using the trust scores as relative weights. The consensus to which this procedure converges is returned as the final collective prediction, denoted $p^\ast(x')$. In the following we discuss some basic properties of the consensus prediction found by DeGroot aggregation.

In the context of combining expert judgement, the seminal work by ~\cite{lehrer81} on social choice theory has defined \emph{unanimity} as a general property any such aggregation function should satisfy. Unanimity requires that  when all agents have the same subjective opinion, the combined prediction should be no different. Algorithms~\ref{alg:DG} naturally satisfies this condition.

\begin{proposition}[Unanimity]\label{prop:unanimity}
If all agents agree on the prediction, then the consensus prediction from Algorithm~\ref{alg:DG} agrees with the prediction of each agent: $p^*(x') = f_i(x')$ for every $i \in [K]$.
\end{proposition}

In addition, our algorithm naturally preserves a global ordering of the models, and the consensus weight of each agent is bounded from below and above by the minimal and maximal trust she receives from any of the other agents.

\begin{proposition}[Ranking-preserving]\label{prop:monotone-row}
Suppose all agent rankings have the same order: for all $j_1, j_2 \in [K]$, if $\tau_{ij_1} \geq \tau_{ij_2}$ for some $i \in [K]$, then $\tau_{i'j_1} \geq \tau_{i'j_2}$ for all $i' \in [K]$. 
Then, Algorithm~\ref{alg:DG} finds a consensus $p^*=\sum_i w_i f_i(x')$ where for pairs $j_1, j_2$ such that $\tau_{ij_1} \geq \tau_{ij_2}$, we have
\begin{equation*}
    w_{j_1} \ge w_{j_2}.
\end{equation*}
\end{proposition}

\begin{proposition}[Bounded by min and max trust]\label{prop:trust-bound}
The final weight assigned to agent $i$ is between the minimal and maximal trust assigned to it by the set of agents: 
\begin{equation*}
    \max_{j} \tau_{ji} \ge w_i \ge \min_{j} \tau_{ji}.
\end{equation*}
\end{proposition}

Finally, Algorithm~\ref{alg:DG} recovers equally-weighted averaging whenever each agent receives a constant amount of combined total trust from the other agents. Thus, if all agents perform equally well on the different validation sets, averaging is naturally recovered. However, the trust not necessarily needs to be allocated uniformly, there may be multiple $T$ that result in a uniform stationary distribution. 

\begin{proposition}[Averaging as a special case]\label{prop:recover-avg} 
If the columns of $T$ sum to 1, then Algorithm~\ref{alg:DG} returns an equal weighting of the agents: $w_i = 1/K$ for $i=1,\dots,K$.
\end{proposition}

Together these properties serve as a basic sanity check that Algorithm~\ref{alg:DG} implements a reasonable consensus-finding procedure, from a decision-theoretic as well as algorithmic perspective.

\subsection{Optimality in a large-sample limit}
In this section we analyze the large-sample behavior of the DeGroot consensus mechanism, showing that it is optimal under the assumption that agents are independent. 

For our large-sample analysis, we suppose that the agents' local datasets $\mathcal{D}_k$ are drawn independently from (different) distributions $\mathcal P_k$ over $\mathcal{X} \times \mathcal{Y}$. In order for our prediction task to be well-defined, we assume that the conditional distribution of $Y$ given $X$ is the same for all $\mathcal P_k$. In other words, the distributions $\mathcal P_k$ are all covariate shifts of each other. For simplicity, we assume the distributions $\mathcal P_k$ are all continuous and have the same support. We will consider the large-sample regime in which each agent has a growing amount of data, growing at the same rate. That is, if $n = |\mathcal{D}_1| + \dots + |\mathcal{D}_k|$ is the number of total training points, then 
$ |\mathcal{D}_k| / n \to c_k > 0,$ as $n\to\infty$ for each agent $k=1,\dots,K$. Lastly, we require a basic consistency condition: 
for any compact set $\mathcal{A} \subset \mathcal{X}$, ${f}_k \to f^*_k$ uniformly over $x \in \mathcal{A}$ as $n \to \infty$, for some continuous function $f^*_k$.

\begin{theorem}[DeGroot converges to inverse-MSE weighting]\label{thm:conv-inverse-mse}
Let $x' \in \mathcal{X}$ be some test point. Assume that $\mathcal P_k$ is supported in some ball of radius $\delta_0$ centered at $x'$, and that the first four conditional moments of $Y$ given $X = x$ are continuous functions of $x$ in this neighborhood.
Next, suppose we run Algorithm~\ref{alg:DG} choosing $N=\lceil n^c \rceil$ nearest neighbors for some $c \in (0,1)$.  Let
\begin{equation*}
    \mathrm{MSE}^*_k = \mathbb{E}\left[(Y - f^*_k(x'))^2\right]
\end{equation*}
denote the asymptotic MSE of model $k$ at the test point $x'$, where the expectation is over a draw of $Y$ from the distribution of $Y \mid X = x'$. Then, the DeGroot consensus mechanism yields weights
\begin{equation*}
    w_k \to w_k^* = \frac{1/\mathrm{MSE}^*_k}{1/\mathrm{MSE}^*_1 + \dots + 1/\mathrm{MSE}^*_K}.
\end{equation*}
\end{theorem}
In other words, in the large-sample limit the DeGroot aggregation yields inverse-MSE weighting and thus recovers inverse-variance weighting for unbiased estimators---a canonical weighting scheme from classical statistics~\citep{hartung2008statistical} that is known to be the optimal way to linearly combine independent, unbiased estimates. This can be  extended to our setting to show that DeGroot averaging leads to the optimal aggregation when we have independent\footnote{While independence is a condition that we can't readily check in practice, it may hold approximately if the agents use different models constructed from independent data. In any case, the primary takeaway is that DeGroot is acting reasonably in this tractable case.} 
agents, as stated next. 

\begin{theorem}[DeGroot is optimal for independent, unbiased agents]\label{thm:degroot_indep_opt}
Assume that the conditional mean and variance of $Y$ given $X = x$ are continuous functions of $x$. For some $\delta > 0$, consider drawing $\tilde{X}$ uniformly from a ball of radius $\delta$ centered at $x'$ and $\tilde{Y}$ from the distribution of $Y \mid X = \tilde{X}$. Under this sampling distribution, suppose the residuals $\tilde{Y} - {f}_k^*(\tilde{X})$ from the agents' predictions have mean zero and each pair has correlation zero. Then the optimal weights,
\begin{equation*}
    \tilde{w} := \argmin_{w \in \mathbb{R}^k: \|w\|_1 = 1} \mathbb{E}_{(\tilde{X}, \tilde{Y})}\bigg[{\Big(\tilde{Y} - \sum_{k\in[K]} w_k {f}_k^*(\tilde{X})\Big)^2}\bigg],
\end{equation*}
approach the DeGroot weights:
$\tilde{w} \to w^*$ as $\delta \to 0$.
\end{theorem}

In summary, Theorem~\ref{thm:conv-inverse-mse} shows that the DeGroot weights are asymptotically converging to a meaningful (sometimes provably optimal) aggregation of the models -- \emph{locally for each test point $x'$}. This is a stronger notion of adaptivity than that usually considered in the model-averaging literature. 

\subsection{Error bars for predictions}

Finally, we develop a decentralized jackknife algorithm~\citep{quenouille1949, Efron1993} 
to estimate the standard error of the consensus prediction $p^*(x')$. 
Our proposed procedure measures the impact of excluding a random agent from the ensemble and returns error bars that measure how stable the consensus prediction is to the observed collection of agents. Formally, let $p^*_{-i}(x')$ denote the consensus reached by the ensemble after removing agent $i$. Then, the jackknife estimate of standard error at $x'$ corresponds to 
\begin{equation}
\widehat{\mathrm{SE}}(x') = \sqrt{\frac{K-1}{K} \sum_{i\in[K]} \left(p^{*}_{-i}(x') - \bar{p}^*(x')\right)^2},
\label{eq:SE}
\end{equation}
where $\bar{p}^*(x') = \frac{1}{K} \sum_{i=1}^K p^{*}_{-i}(x')$ is the average delete-one prediction.

In the collective prediction setting, this is an attractive target of inference because it can be computed in a decentralized manner, is entirely agnostic to the data collection or training mechanism employed by the agents, and requires no additional model evaluations above those already carried out in Algorithm~\ref{alg:DG}. Furthermore, it allows one to diagnose the impact of a single agent on the collective prediction and thus assess the vulnerability of the algorithm to malicious agents.
A detailed description of the DeGroot jackknife procedure can be found in Algorithm~\ref{alg:DG_jackknife} in Appendix~\ref{app:standard_errors}.

\section{Experiments}
We investigate the efficacy of the DeGroot aggregation mechanism empirically for various datasets, partitioning schemes and model configurations. We start with a synthetic setup to illustrate the strengths and limitations of our method. Then, we focus on real datasets to see how these gains surface in more practical applications with natural challenges, including data scarcity, heterogeneity of different kinds, and scalability. We compare our Algorithm~\ref{alg:DG} (DeGroot) to the natural baseline of equally-weighted model averaging (M-avg) and also compare to the performance of individual models composing the ensemble. In Section~\ref{sec:realdata} we include an additional comparison with two reference schemes that have access to shared validation data and a central agent that determines an (adaptive) weighting scheme---information DeGroot does not have.

\subsection{Synthetic data}
\label{sec:synthetic}

First, we investigate a synthetic two-dimensional setting where each agent's features are drawn from a multivariate Gaussian distribution, $x\sim\mathcal N(\mu_k,\Sigma_k)$.
The true labeling function is given as $y=[1+\text{exp}({\alpha^\top x})]^{-1}$ and we add white noise of variance $\sigma_Y^2$ to the labels in the training data. 
Unless stated otherwise, we use $K=5$ agents and let each agent fit a linear model to her local data. See Figure~\ref{fig:linearmodels} in the Appendix for a visualization of the models learned by the individual agents, and how they accurately approximate the true labeling function in different regions of the input space.\footnote{If not stated otherwise we spread the means as $\mu=\{[\text{-}3,\text{-}4],[\text{-}2,\text{-}2],[\text{-}1,\text{-}1],[0,0],[3,2]\}$ and let $\Sigma_k =\text{I}$. We use 200 training samples on each agent, for the labeling function we use $\alpha=[1,1]$, and for the label noise in the training data we set $\sigma_Y =0.1$. We use $N=5$ for local cross-validation in DeGroot.}

\begin{figure}[t!]
    \centering
    \subfloat[{\centering individual test points}]{\includegraphics[height=0.24\textwidth]{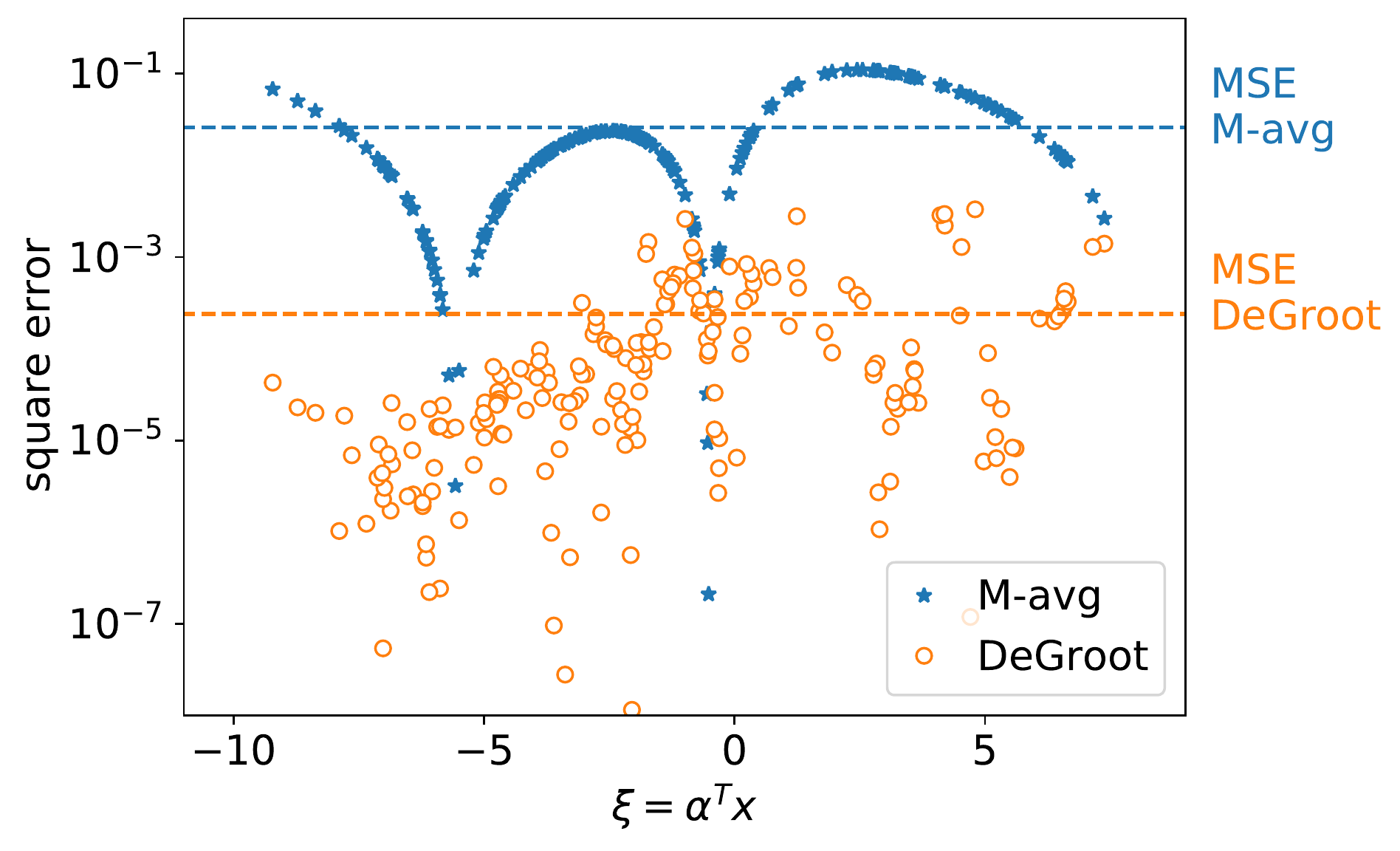}\label{fig:testerror}}\hfill
    \subfloat[consensus finding]{\includegraphics[height=0.24\textwidth]{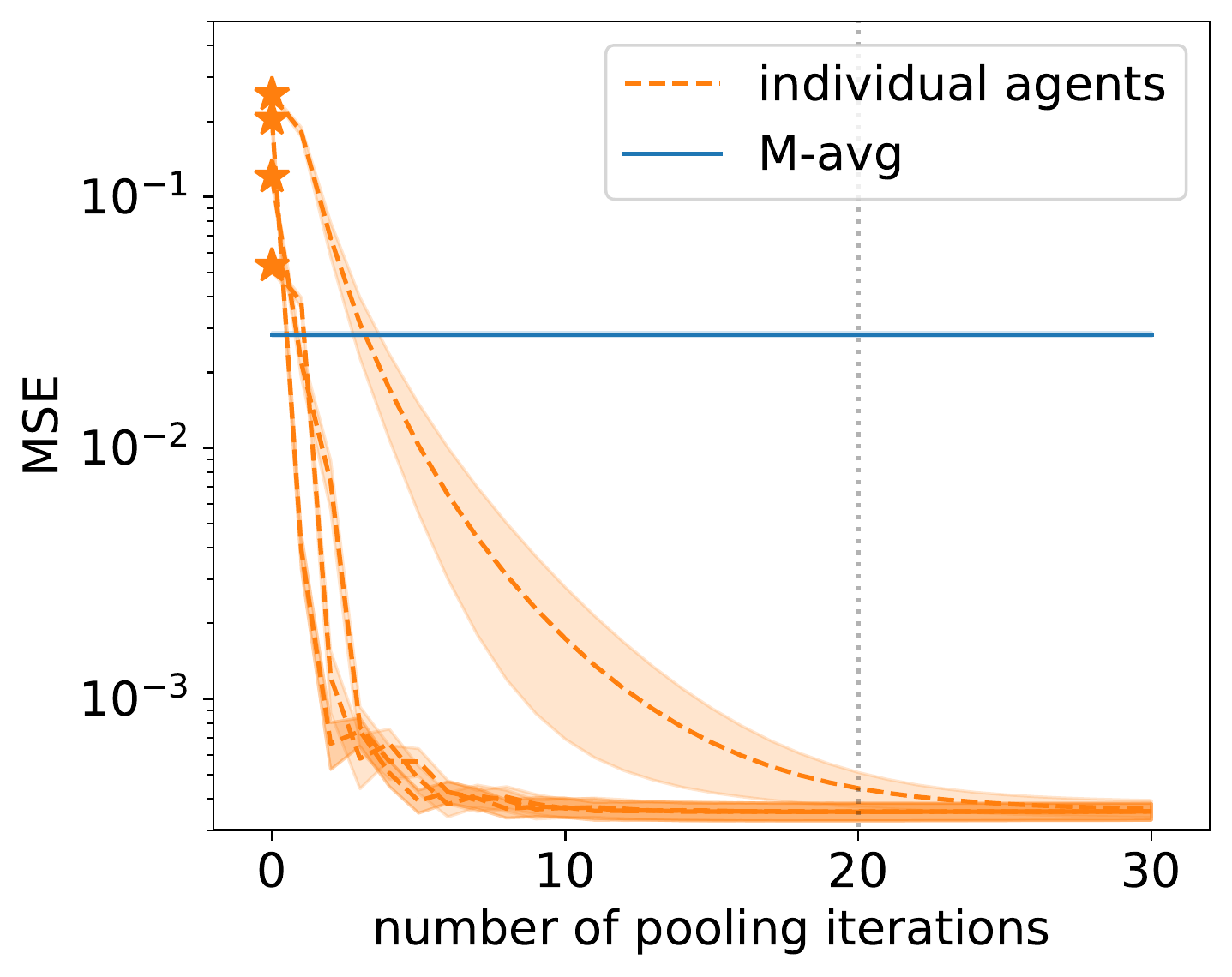}\label{fig:consensus}}
    \subfloat[DeGroot jackknife]{\includegraphics[height=0.24\textwidth]{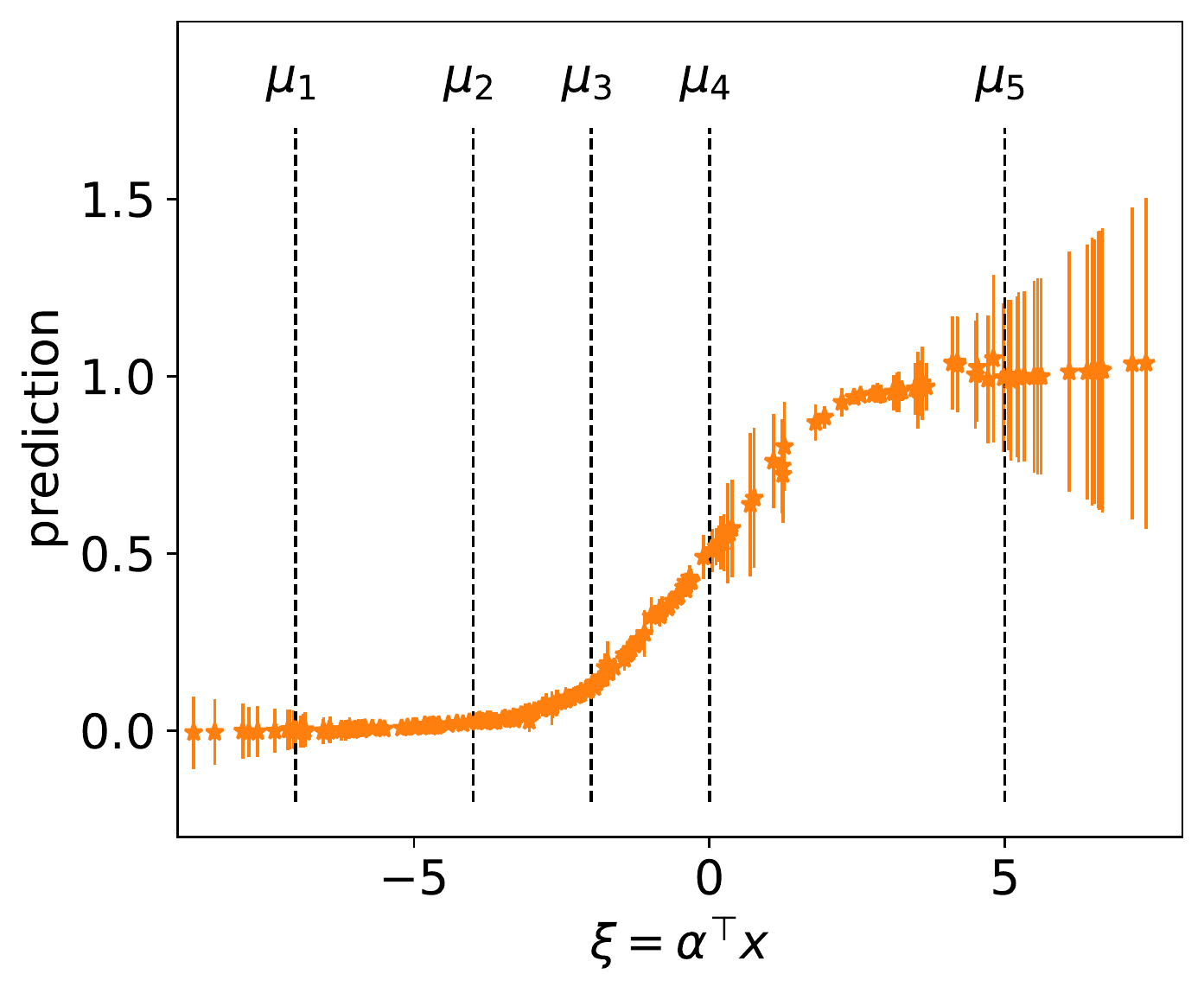}\label{fig:errorbars}}
     \caption{\emph{(Synthetic data experiment)}.  Comparison of DeGroot to M-avg. (a) Predictive accuracy on individual test points. (b) Performance for varying number of pooling iterations during DeGroot consensus finding. (c) Confidence intervals for individual test points returned by DeGroot jackknife. 
    }
    \label{fig:local-fit}
\end{figure}

To evaluate the DeGroot aggregation procedure we randomly sample 200 test points from the uniform mixture distribution spanned by the $K$ training data distributions, and compare the prediction of DeGroot to M-avg. We visualize the squared error for the collective prediction on the individual test points in Figure~\ref{fig:testerror} where we use $\xi:=\alpha^\top x$ as a summary for the location of the test point on the logistic curve. We observe that DeGroot consistently outperforms M-avg across the entire range of the feature space. Exceptions are three singularities, where the errors of the individual agents happen to cancel out by averaging. 
Overall DeGroot achieves an MSE of $4.6e^{-4}$ which is an impressive reduction of over $50\times$ compared to M-avg. Not shown in the figure is the performance of the best individual model which achieves an MSE of $5e^{-2}$ and performs worse than the M-avg baseline. Thus, DeGroot outperforms an oracle model selection algorithm and every agent strictly benefits from participating in the ensemble.
The power of adaptive weighting such as used in DeGroot is that it can trace out a nonlinear function, given only linear models, whereas any static weighting (or model selection) scheme will always be bound to the best linear fit.

In Figure~\ref{fig:consensus} we show how individual agent's predictions improve with the number of pooling iterations performed in Step 9 of Algorithm~\ref{alg:DG}. The influence of each local model on the consensus prediction of $x'$ and how it changes depending on the location of the test point is illustrated in Figure~\ref{fig:teaser} in the introduction. 

An interesting finding that we expand on in Appendix~\ref{app:synthetic} 
is that for the given setup the iterative approach of DeGroot is superior to a more naive approach of combining the individual trust scores or local MSE values into a single weight vector. In particular, DeGroot reduces the MSE by $20\times$ compared to using the average trust scores as weights.

In Figure~\ref{fig:errorbars} we visualize the error bars returned by our DeGroot jackknife procedure. As expected, the intervals are large in regions where only one a small number of agents possess representative training data, and small in regions where there is more redundancy across agents.

Finally, in Appendix~\ref{app:synthetic} 
we conduct further qualitative investigations of the DeGroot algorithm. First, we vary the covariance in the local data feature distribution. We find that the gain of DeGroot over M-avg becomes smaller as the covariance is increased; models become less specialized and there is little for DeGroot to exploit. On the other extreme, if the variance in the data is extremely small, the adaptivity mechanism of DeGroot becomes less effective, as the quality of individual models can no longer be reliably assessed. However, we need to go to extreme values in order to observe these phenomena.
In a second experiment we explore DeGroot's sensitivity to choices of $N$. We find that when $N$ is chosen too large adaptivity can suffer, and for very small values of $N$ performance can degrade if the labels are highly noisy. However, there is a broad spectrum of values for which the algorithm performs well, between $1\%$ and $10\%$ of the local data usually a good choice.

\subsection{Real datasets}
\label{sec:realdata}

\begin{figure}[t!]
    \centering
    \subfloat[label heterogeneity]{\includegraphics[width=0.33\textwidth]{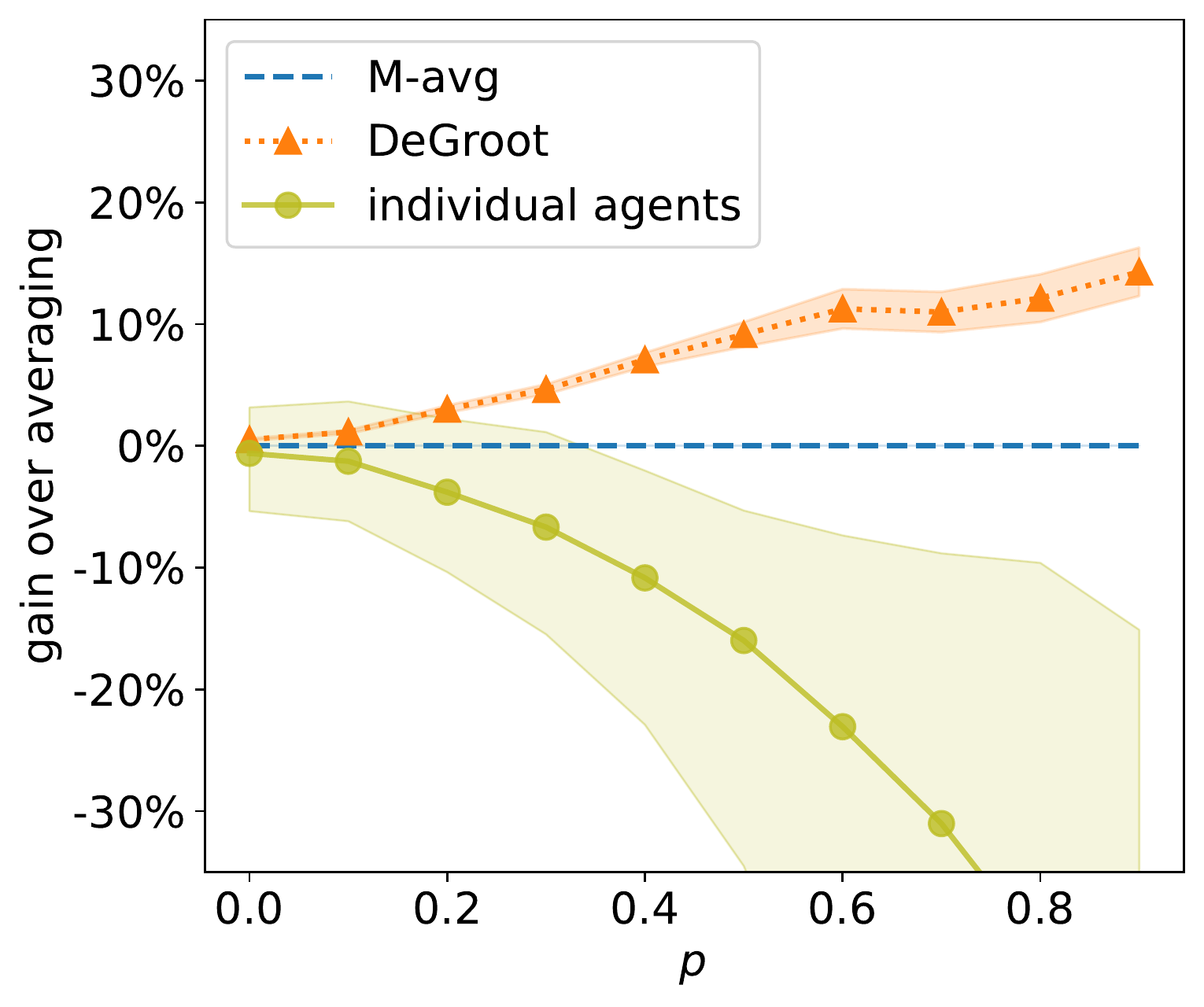}\label{fig:heter1}
    }
    \subfloat[feature heterogeneity]{\includegraphics[width=0.33\textwidth]{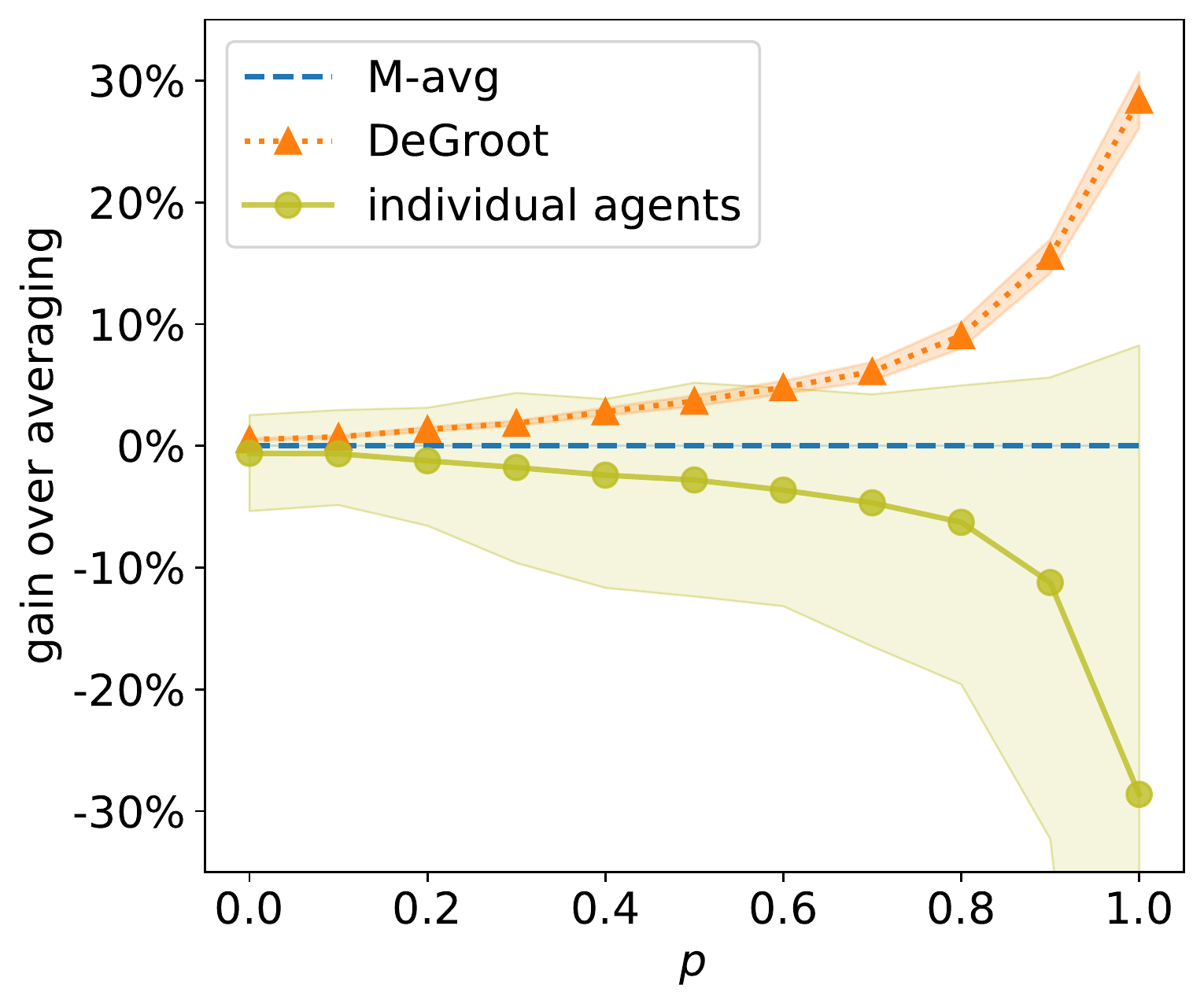}\label{fig:heter3}
    }
    \subfloat[regularizer heterogeneity]{\includegraphics[width=0.33\textwidth]{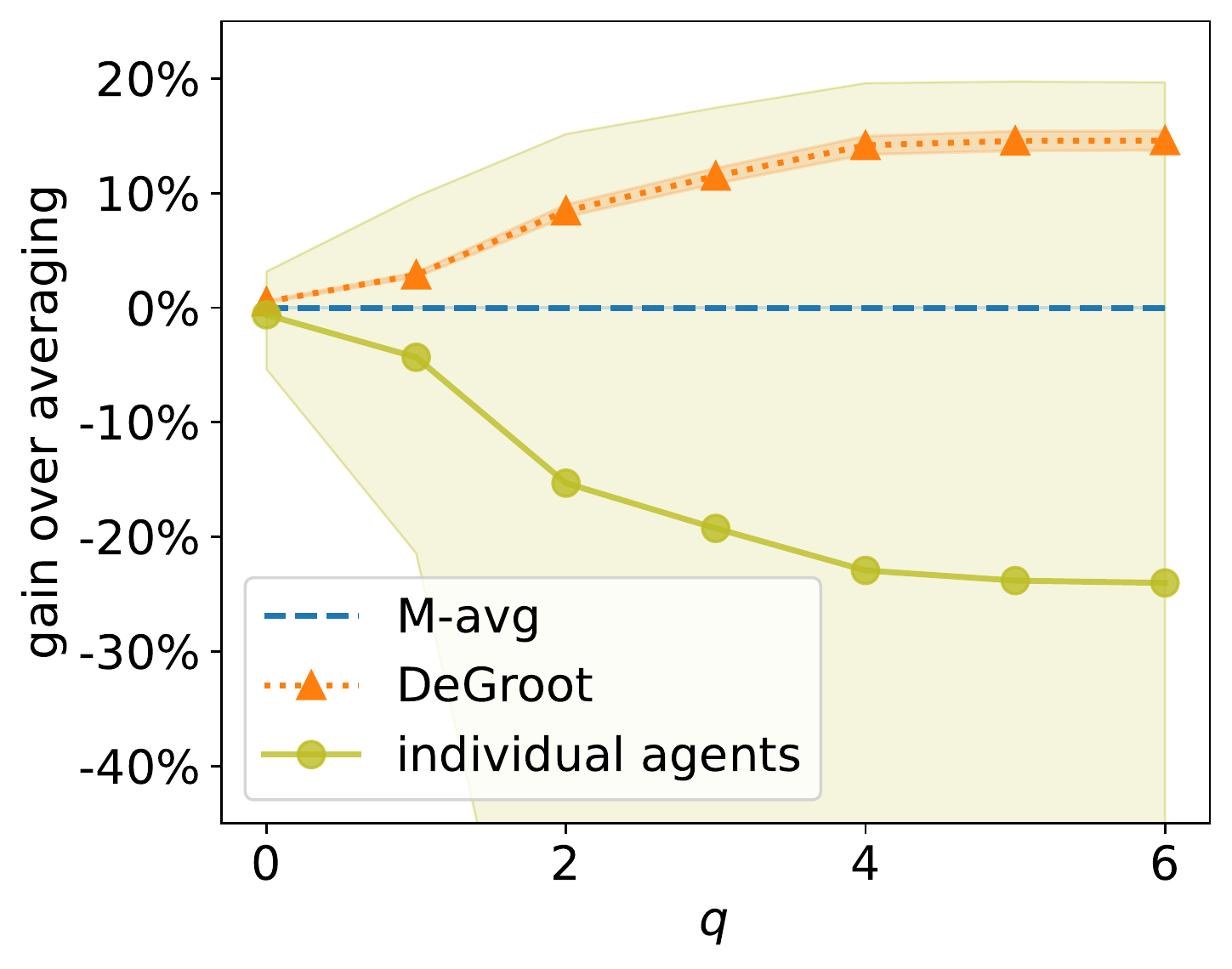}\label{fig:heter2}}
     \caption{\emph{(Data and model heterogeneity for the abalone data).} Relative gain over M-avg. Confidence intervals of DeGroot are over randomness in data partitioning and for the individual agents the shaded area spans the performance from the worst to the best model, the line indicating average performance. \vspace{-0.2cm}}
    \label{fig:heter}
\end{figure}

Turning to real data sets, we first investigate how DeGroot deals with three different sources of heterogeneity. We work with the abalone dataset~\citep{abalonedata} and train a lasso model on each agent\footnote{Similar phenomena can be observed for other datasets and models (see Appendix~\ref{app:realdata})} 
with regularization parameter $\lambda_k=\lambda'$ that achieves a sparsity of  ${\scriptsize \sim}0.8$.  
For our first experiment, shown in Figure~\ref{fig:heter1}, we follow a popular setup from federated learning to control heterogeneity in the data partitioning~\citep[see, e.g.,][]{fed21shen}. We partition the training data in the following manner: 
a fraction $1-p$ of the training data is partitioned randomly, and a fraction $p$ is first sorted by the outcome variable $y$ and then partitioned sequentially.  
For the second experiment, shown in Figure~\ref{fig:heter3}, we use a similar partitioning scheme, but sort the data along an important feature dimension, instead of $y$. This directly leads to heterogeneity in the input domain. Finally, we investigate model heterogeneity arising from different hyper-parameter choices of the regularization parameter $\lambda$, given a random partitioning. To control the level of heterogeneity, we start with $\lambda = \lambda'$ on each agent, and let the parameters diverge increasingly, such that $\lambda_k =\lambda' \left[1+\frac{k-k_0}{K}\right]^q$ with $k_0=3$. The results are depicted in Figure~\ref{fig:heter2}.
In all three experiments the gain of DeGroot over M-avg is more pronounced with increasing degree of heterogeneity. When there is heterogeneity in the partitioning, DeGroot clearly outperforms the best individual model, by taking advantage of models' local areas of expertise. For the third experiment where the difference is in model quality, DeGroot almost recovers the performance of the best single model.

\begin{table}[]
\footnotesize{
    \centering
    \caption{\emph{(Benchmark comparison).} Performance for different combinations of datasets and models. Datasets have been downloaded from~\citep{libsvm}. We use ridge and lasso as a baseline and train a decision tree regressor (DTR) for discrete outcome variables, and a neural net (NN) with two hidden layers for continuous outcomes. A positive relative gain means a reduction in MSE compared to DeGroot. }
\begin{tabular}{@{}llC{40pt}rrrC{63pt}}
\toprule
         \multicolumn{2}{c}{ } & \multirow{2}{*}{\shortstack[c]{MSE\\DeGroot}} & \multicolumn{3}{c}{gain relative to DeGroot [\%]}&  \multirow{2}{*}{\shortstack[c]{hyper-\\parameters}}\\ 
         \cmidrule{4-6}
         \multicolumn{2}{c}{ }  &   & \multicolumn{1}{c}{M-avg} & \multicolumn{1}{c}{CV-static}& \multicolumn{1}{c}{CV-adaptive}&\\
          \midrule
         Boston &Ridge&25.23&\color{negvals}{-12.45$\pm$2.29}&\color{negvals}{-10.24$\pm$1.90}&\color{negvals}{-2.80$\pm$1.72}&$\lambda = 1e^{-5}$\\
           &Lasso&24.17&\color{negvals}{-13.70$\pm$1.72}&\color{negvals}{-10.35$\pm$1.16}&\color{negvals}{-3.46$\pm$1.44}&$\lambda=5e^{-3}$\\
          &NN&14.18&\color{negvals}{-15.18$\pm$3.43}&\color{negvals}{-11.01$\pm$3.08}&\color{negvals}{-5.81$\pm$3.20}&layer$ = (7,7)$\\
          \midrule
         E2006 &Ridge&0.15&\color{negvals}{-0.12$\pm$0.06}&\color{negvals}{-0.21$\pm$0.06}&\color{posvals}{0.13$\pm0.11$}&$\lambda = 5e^{-2}$\\
         &Lasso&0.097&\color{negvals}{-0.92$\pm$0.15}&\color{negvals}{-0.90$\pm$0.13}&\color{negvals}{-0.08$\pm$ 0.07}&$\lambda=5e^{-5}$ 
         \\
         &NN& 0.11 &\color{negvals}{-14.58$\pm$3.48}&\color{negvals}{-5.69$\pm$0.97}&\color{neutralvals}{-0.16$\pm0.53$}&layer$ = (9,9)$ \\
         \midrule
        Abalone & Ridge &3.18&\color{negvals}{-3.67 $\pm$ 0.55}&\color{negvals}{-3.52$\pm$0.47}& \color{negvals}{-0.49$\pm$0.46}&$\lambda = 5e^{-2}$\\
         & Lasso &4.93& \color{negvals}{-10.09$\pm$ 0.83}&\color{negvals}{-10.05$\pm$0.80}&\color{neutralvals}{-0.39$\pm$0.75}&$\lambda=5e^{-2}$\\
         & DTR &4.85&\color{negvals}{-2.53$\pm$ 0.69} &\color{negvals}{-2.55$\pm$0.70}&\color{neutralvals}{-0.79$\pm$0.88}& max\_depth = 4\\
         \midrule
         cpusmall &Ridge&59.00&\color{negvals}{-96.23$\pm$23.46}&\color{negvals}{-89.81$\pm$15.58}&\color{negvals}{-0.78$\pm$0.63}&$\lambda=1e^{-5}$\\
          &Lasso&53.82&\color{negvals}{-76.08$\pm$14.45}&\color{negvals}{-73.81$\pm$10.51}&\color{negvals}{-2.04$\pm$1.82}&$\lambda=1e^{-3}$\\
          &DTR&11.16&\color{negvals}{-4.05$\pm $1.88}&\color{negvals}{-3.90$\pm$1.81}&\color{neutralvals}{1.65$\pm$2.20}&max\_depth=7\\
          \midrule
          YearPrediction&Ridge& 95.02&\color{negvals}{-0.63$\pm$0.24} & \color{negvals}{-0.94$\pm$0.11}& \color{posvals}{0.21$\pm$0.12}& $\lambda = 1$ \\
          &Lasso& 91.21&\color{negvals}{-1.73$\pm$0.47} & \color{negvals}{-1.98$\pm$0.35}& \color{neutralvals}{-0.22$\pm$0.29}& $\lambda = 1$ \\
          &DTR&62.11&\color{negvals}{-3.13$\pm$1.83}&\color{negvals}{-2.26$\pm$1.02}&\color{negvals}{-0.54$\pm$0.53}&max\_depth = 4\\
          \bottomrule
    \end{tabular}
    \label{tab:benchmark}}
    \vspace{-0.2cm}
\end{table}

Next, we conduct a broad comparison of DeGroot to additional reference schemes across various datasets and models, reporting the results in Table~\ref{tab:benchmark}. To the best of our knowledge there is no existing generic adaptive aggregation scheme that does not rely on external validation data to either build the aggregation function or determining the weights $w$. Thus, we decided to include a comparison with an approach that uses a validation dataset to compute the weighting, where 
the external validation data has the same size as the local partitions. Our two references are static inverse-MSE weighting (CV-static) where the weights are determined based on the model's average performance on the validation data, and adaptive inverse-MSE weighting (CV-adaptive), where the performance is evaluated only in the region around the current test point $x'$. We achieve heterogeneity across local datasets by partially sorting a fraction $p=0.5$ of the labels, and we choose $N$ to be $1\%$ of the data partition for all schemes (with a had lower bound at $2$). Each scheme is evaluated for the same set of local models and confidence intervals are over the randomness of the data splitting into partitions, test and validation data across evaluation runs. The results demonstrate that DeGroot can effectively take advantage of the proprietary data to find a good aggregate prediction, significantly outperforming M-avg and CV-static, and achieving better or comparable performance to CV-adaptive that works with additional data.

Finally, we verify that our aggregation mechanism scales robustly with the number of agents in the ensemble and the simultaneously decreasing partitioning size. Results in Appendix~\ref{app:realdata} 
show that DeGroot consistently outperforms model averaging on two different datasets on the full range from 2 up to 128 agents.

\section{Discussion}

We have shown that insights from the literature on humans consensus via discourse suggest an effective aggregation scheme for machine learning. In settings where there is no additional validation data available for assessing individual models quality, our approach offers an appealing mechanism to take advantage of individual agents' proprietary data to arrive at an accurate collective prediction.
We believe that this approach has major implications for the future of federated learning since it circumvents need to exchange raw data or model parameters between agents.  Instead, all that is required for aggregation is query access to the individual models.

\section*{Acknowledgements}
We would like to thank Jacob Steinhardt and Alex Wei for feedback on this manuscript.


{
\bibliography{ref}
\bibliographystyle{plainnat}
}
\clearpage

\newpage
\appendix

\appendix

\section{Proofs}
\label{app:proofs}
For some of our proofs we will use a Markov chain interpretation of the DeGroot mechanism, where the matrix $T$ defines the transition probabilities of the Markov chain. By definition of the trust scores, the rows of $T$ sum to one. Moreover, the entries are all positive, so this is an aperiodic, irreducible Markov chain and hence has a unique stationary distribution, given by the solution $w$ to the linear system $w = Tw$~\citep[e.g.,][]{haggstrom2002}. Furthermore, $T$ has a unique largest left-eigenvector with eigenvalue $1$. It will be convenient for us to consider a Markov chain $Z_0, Z_1,\dots$ with transition matrix $T$, and $Z_0$ drawn from the stationary distribution $w$.

\begin{proof}[Proof of Proposition~\ref{prop:unanimity}]
We prove the claim by induction on the number of consensus steps $t$. First, we have that $p_i^{(0)} = f_1(x')$ for all $i \in [K]$ by assumption.  Suppose $p_i^{(t)} = f_1(x')$ for all $i \in [K]$. Then, \[p_i^{(t+1)} = \sum_{j=1}^K \tau_{i,j}p_j^{(t)} = f_1(x')\] 
since the weights sum to 1.
\end{proof}

\begin{proof}[Proof of Proposition~\ref{prop:monotone-row}]
Using our Markov chain notation above, it suffices to show that 
\[P(Z_1 = j_1) \ge P(Z_1 = j_2).\] 
It is easy to see that this is the case since 
\[P(Z_1 = j_1 \mid Z_0 = i) \ge P(Z_1 = j_2 \mid Z_0 = i)\] 
for all $i \in [K]$.
\end{proof}

\begin{proof}[Proof of Proposition~\ref{prop:trust-bound}]
Using our Markov chain notation above, note that
\begin{equation*}
\max_j \{P(Z_1 = i \mid Z_0 = j)\} \ge P(Z_1 = i) \ge \max_j \{P(Z_1 = i \mid Z_0 = j)\}.
\end{equation*}
The result follows.
\end{proof}

\begin{proof}[Proof of Proposition~\ref{prop:recover-avg}]
Note that in this case, $1^\top T = 1$ (where $1$ is the $K$-vector of all ones), so $w = 1 / n$ is the stationary state. 
\end{proof}

\begin{proof}[Proof of Theorem~\ref{thm:conv-inverse-mse}]
Choose any $i \in [K]$. We will show that $\tau_{ij} \to 1 / \mathrm{MSE}^*_j$ in probability as $n \to \infty$. 

First, let $x_1,\dots,x_N \in \mathcal{D}_i$ be the $N$ nearest neighbors to the test point $x'$ in agent $i$'s data. We claim that
\begin{equation}\label{eq:points_get_close}
\max_{m \in \{1,\dots,N\}} \|x_m - x'\| \to 0
\end{equation}
in probability as $n \to \infty$. To see this, for $\delta_0 > \delta > 0$, note that $P(\|x_m - x'\| < \delta) = \epsilon$ for some $\epsilon > 0$, since the points in $\mathcal{D}_i$ are an i.i.d. sample from a distribution supported in this ball. But then,
\[P\left(\max_{m \in \{1,\dots,N\}} \|x_m - x'\| \ge \delta\right) = P(B < N)\]
where $B$ is distributed as a binomial with $|\mathcal{D}_i|$ draws and success probability $\epsilon$. But
\begin{align*}
    P\left(B < N\right) &= P\left(\frac{B - \epsilon |\mathcal{D}_i|}{\sqrt{|\mathcal{D}_i|} \epsilon (1 - \epsilon)} <  \frac{N - \epsilon |\mathcal{D}_i|}{\sqrt{|\mathcal{D}_i|} \epsilon (1 - \epsilon)}\right) \\
\end{align*}
By assumption the term ${N - \epsilon |D_i|} / {\sqrt{|\mathcal{D}_i|} \epsilon (1 - \epsilon)} \to -\infty$ as $n \to \infty$ since $N / |\mathcal{D}_i| \to 0$. Then, by the central limit theorem we conclude that
\[P\left(B < N\right) \to 0\] 
as $n \to \infty$, and hence~\eqref{eq:points_get_close} holds.

Next, let $\mu(x) = \E[Y \mid X = x]$ and $\sigma(x) = \textrm{Var}(Y \mid X = x)$ be the true mean and variance functions, respectively. For convenience, let $B_\delta(x')$ be the ball of radius $\delta$ centered at $x'$. We have
\begin{multline*}
      \inf_{x \in B_\delta(x')} \{\mu(x) - f_j(x))^2 + \sigma(x)^2\} \le  \E\left[1 / \tau_{i,j} \big| \max_{m \in \{1,\dots,N\}} \|x_m - x'\| < \delta\right]
      \le \sup_{x \in B_\delta(x')} \{\mu(x) - f_j(x))^2 + \sigma(x)^2\}. 
\end{multline*}
Note that the lower and upper bound both converge to $\textrm{MSE}^*_j$ as $\delta \to 0$, by the continuity of $\mu(x), f_j(x)$ and $\sigma(x)$. Similarly, the conditional variance of $1 / \tau_{i,j}$ is bounded by $C / N$ for some constant $C$. Thus, by Chebyshev's inequality, there exist sequences $c_1(n) \to 0$ and $c_2(n) \to 0$ such that 
\begin{equation*}
    P\left(|\tau_{ij} - \textrm{MSE}^*_j| \ge c_1(n) \mid \max_{m \in \{1,\dots,N\}} \|x_m - x'\| < \delta \right) \le c_2(n).
\end{equation*}
Since the event $\{\max_{m \in \{1,\dots,N\}} \|x_m - x'\| < \delta\}$ has probability tending to $1$ as $n$ grows, this implies that $1 / \tau_{i,j} \to \textrm{MSE}^*_j$ in probability, as desired.
\end{proof}

\begin{proof}[Proof of Theorem~\ref{thm:degroot_indep_opt}]
For this distribution, we wish to find weights $\tilde{w}$ solving the following:
\begin{equation*}
    \tilde{w} := \argmin_{w \in \mathbb{R}^k, 1^\top w = 1} \mathbb{E}_{(\tilde{X}, \tilde{Y})}\left[\left(\tilde{Y} - \sum_{k=1}^K w_k {f}_k^*(\tilde{X})\right)^2\right].
\end{equation*}
Expanding the right side, we have
\begin{align*}
    \mathbb{E}_{(\tilde{X}, \tilde{Y})}\left[\left(\tilde{Y} - \sum_{k=1}^K w_k {f}_k^*(\tilde{X})\right)^2\right] &= \mathbb{E}_{(\tilde{X}, \tilde{Y})}\left[\left(\sum_{k=1}^K w_k (\tilde{Y} - {f}_k^*(\tilde{X}))\right)^2\right] \\
    &= \textrm{Var}\left(\sum_{k=1}^K w_k (\tilde{Y} - {f}_k^*(\tilde{X}))\right) \\
    &=  \sum_{k=1}^K \textrm{Var}\left( w_k (\tilde{Y} - {f}_k^*(\tilde{X}))\right) \\ 
    &\quad+ 2 \sum_{1 \le k_1 < k_2 \le K} \textrm{Cov}\left( w_{k_1} (\tilde{Y} - {f}_{k_1}^*(\tilde{X}),  w_{k_2} (\tilde{Y} - {f}_{k_2}^*(\tilde{X}))\right) \\
    &=  \sum_{k=1}^K w_k^2 \textrm{Var}\left(\tilde{Y} - {f}_k^*(\tilde{X})\right) \\ 
\end{align*}
With this expression, one can easily verify that the optimal $w$ is then
\begin{equation*}
    \tilde{w}_k = 1 / \textrm{Var}\left(\tilde{Y} - {f}_k^*(\tilde{X})\right).
\end{equation*}

Lastly, we claim that
\begin{equation*}
    \textrm{Var}\left(\tilde{Y} - {f}_k^*(\tilde{X})\right) \to \textrm{MSE}^*_k
\end{equation*}
as $\delta \to 0$. This follows from the decomposition
\begin{equation*}
    \textrm{Var}\left(\tilde{Y} - {f}_k^*(\tilde{X})\right) =  \E \left[ \textrm{Var}\left(\tilde{Y} - {f}_k^*(\tilde{X}) \mid \tilde{X}\right)\right]  + \textrm{Var}\left(\E[\tilde{Y} - {f}_k^*(\tilde{X}) \mid \tilde{X}]\right).
\end{equation*}
By assumption, the conditional mean and variance are continuous functions of $x$, so the first term in the sum converges to $\textrm{MSE}^*_k$ as $\delta \to 0$. Likewise, the second term in the sum converges to $0$. This completes the proof.

\end{proof}

\section{Estimating the Standard Error}
\label{app:standard_errors}

In this section, we develop a standard errors estimator for the collective prediction $p^*(x')$ resulting from Algorithm~\ref{alg:DG}. Our proposal is a special case of jackknife standard error estimates, where we consider the agents to be the independent samples from a super-population of agents (this probabilistic interpretation is useful to precisely discuss standard errors but is not needed elsewhere in our work). Our error bars will thus measure how stable the consensus prediction is to the observed collection of agents. In the collective prediction setting, this is an attractive target of inference because it can be estimated without requiring any assumptions about how the agents gather data, their models and training procedure.  As a result, our standard error estimate is entirely agnostic to the behavior of the agents,  is decentralized, and requires minimal communication---it requires no additional model queries above what was already done in Algorithm~\ref{alg:DG}.

\begin{algorithm}[!h]
\begin{algorithmic}[1]
\caption{DeGroot Jackknife}\label{alg:DG_jackknife}
\STATE \textbf{Input:} Pre-trained $K$ agents $f_1, \cdots, f_K$ with local training dataset $\mathcal D_k, k \in [K]$; \\ neighborhood size $m$; test point $x'$.\\[1ex]
\STATE Construct trust matrix $T$ as in Algorithm~\ref{alg:DG}.\\[1ex]
\FOR{$i=1,2,\ldots, K$}
\STATE Create submatrix $T^{(i)}$ \\
\STATE Find $v$ such that $v T^{(i)} = v$ by power iteration.\\
\STATE Form collective prediction $p^{*}_{-i}(x') = \sum_{j\ne i} v_j {f}_j(x')$ \\
\ENDFOR\\[1ex]
\STATE \textbf{Return:} Jackknife estimate of standard error at $x'$: 
$$\widehat{\mathrm{SE}}(x') = \sqrt{\frac{K-1}{K} \sum_{i=1}^K \left(p^{*}_{-i}(x') - \bar{p}^*(x')\right)^2},$$
where $\bar{p}^*(x') = \frac{1}{K} \sum_{i=1}^K p^{*}_{-i}(x')$ is the average delete-one prediction at $x'$.
\\[1ex]
\end{algorithmic}
\end{algorithm}

Turning to the details, recall that $T$ is the matrix of trust scores, and let $T^{(i)}$ be the principal submatrix of $T$ formed by deleting row $i$ and column $i$, renormalized so that each row sums to one. That is, $T^{(i)}$ is the trust matrix if agent $i$ is removed from our collection of agents. The idea is that we run DeGroot aggregation with agent $i$ deleted; i.e., we find the collective prediction $p^{*}_{-i}(x')$ of the remaining $K-1$ agents. Then, by looking at the variability of these predictions with different agents removed, we can quantify the stability of the procedure. Specifically, we take the sample standard deviation of these quantities, scaled by $\sqrt{(K-1)^2/K}$: the scaling prescribed by the theory of the jackknife estimator~\citep{Efron1993}. We state this procedure formally in Algorithm~\ref{alg:DG_jackknife}.

Importantly, this calculation does not require any additional model evaluations above what is required to perform the consensus prediction; this algorithm only requires constructing the matrix $T$ and knowing the predictions $f_j(x')$, which were already needed in Algorithm~\ref{alg:DG}. Thus, estimating the standard error has no additional information-sharing requirement; it requires only a modest extra amount of computation.

See Figure~\ref{fig:jack} for an empirical demonstration of these error bars on our synthetic example from Section~\ref{sec:synthetic}.

\section{Details on algorithm implementation}

In this section we discuss some practical considerations when implementing Algorithm~\ref{alg:DG}. 

\textit{Step 4.} For our experiments we have chosen the Euclidean distance metric, and a fixed number $N$ of datapoints to determine $\mathcal{D}_i(x')$. To find the $N$ nearest neighbors we have used the \texttt{sklearn} unsupervised \texttt{NearestNeighbors} classifier from \texttt{sklearn.neighbors}. Note that this step could be adjusted to use other distance metrics specific to the modality of the data. In particular, for higher dimensional data such as images a careful choice of distance metric will become crucial.

\textit{Step 9.} In this step the agents iteratively refine their predictions by pooling other agents predictions to reach a consensus, as described in Section~\ref{sec:DG}. These pooling iterations can be executed in a fully decentralized manner, or the consensus weights can be evaluated via the power method at one of the nodes that gets access to all the trust scores in order to construct the matrix $T$ from $\{\tau_{i,j}\}_{i,j\in[K]}$. The best implementation will depend on the requirements of your application. For our experiments we have chosen the centralized evaluation for convenience. 
Around $t_p=30$ power iterations are usually far sufficient to reach a consensus. Given that the number of agents is usually small ($<1000$) this operation comes with little computational cost and we recommend to pick the number of iterations $t_p$ with sufficient margin to take advantage of the full potential of DeGroot aggregation.
Alternatively, one might also implement a stopping criteria in the form of $|p_j^{(t)}-p_j^{(t-1)}|\leq \text{tol}$ that is checked independently by each agent $j\in[K]$ after every round. 

\emph{Step 10.} For every test point the algorithm returns the consensus prediction $p^*(x')$. If Step 9 has fully converged we have $p_j=p^*$ for all $j$ and we can return any of the individual agents predictions. However, to get a robust procedure, even if the individual beliefs have not fully converged, we return $p^*=\frac 1 K \sum_j p_j$ instead.

\section{Additional Experimental Details and Evaluations}

The code is written in Python and for model training, as well as the nearest neighbor computation we use the built-in functionalities of Sklearn~\citep{scikit-learn}. 

\subsection{Evaluation on Synthetic Data}
\label{app:synthetic}
For the synthetic setup the true label is determined by a logistic function, as described in Section~\ref{sec:synthetic}, where we choose $\alpha = [1,1]$. The features of agent $k$ are distributed according to a multi-variate Gaussian distribution with mean $\mu_k$ and each model then fits a linear function to it's training data. We use the following default experimental configuration: we choose the means as  $\mu_1=[\text{-}3,\text{-}4]$, $\mu_2=[\text{-}2,\text{-}2]$, $\mu_3=[\text{-}1,\text{-}1]$, $\mu_4=[0,0]$ $\mu_5=[3,2]$, the covariance of the local data as $\Sigma_k = \Sigma =\text{I}$ and the variance in the label noise as $\sigma_Y =0.1$.
The resulting local linear fits are illustrated in Figure~\ref{fig:linearmodels} where we see that individual models approximate the true labeling function in different regions of the input space. The accuracy as a function of $\xi$ is illustrated in Figure~\ref{fig:teaser} in the introduction.

\begin{figure}[h!]
    \centering 
    \includegraphics[width=0.6\textwidth]{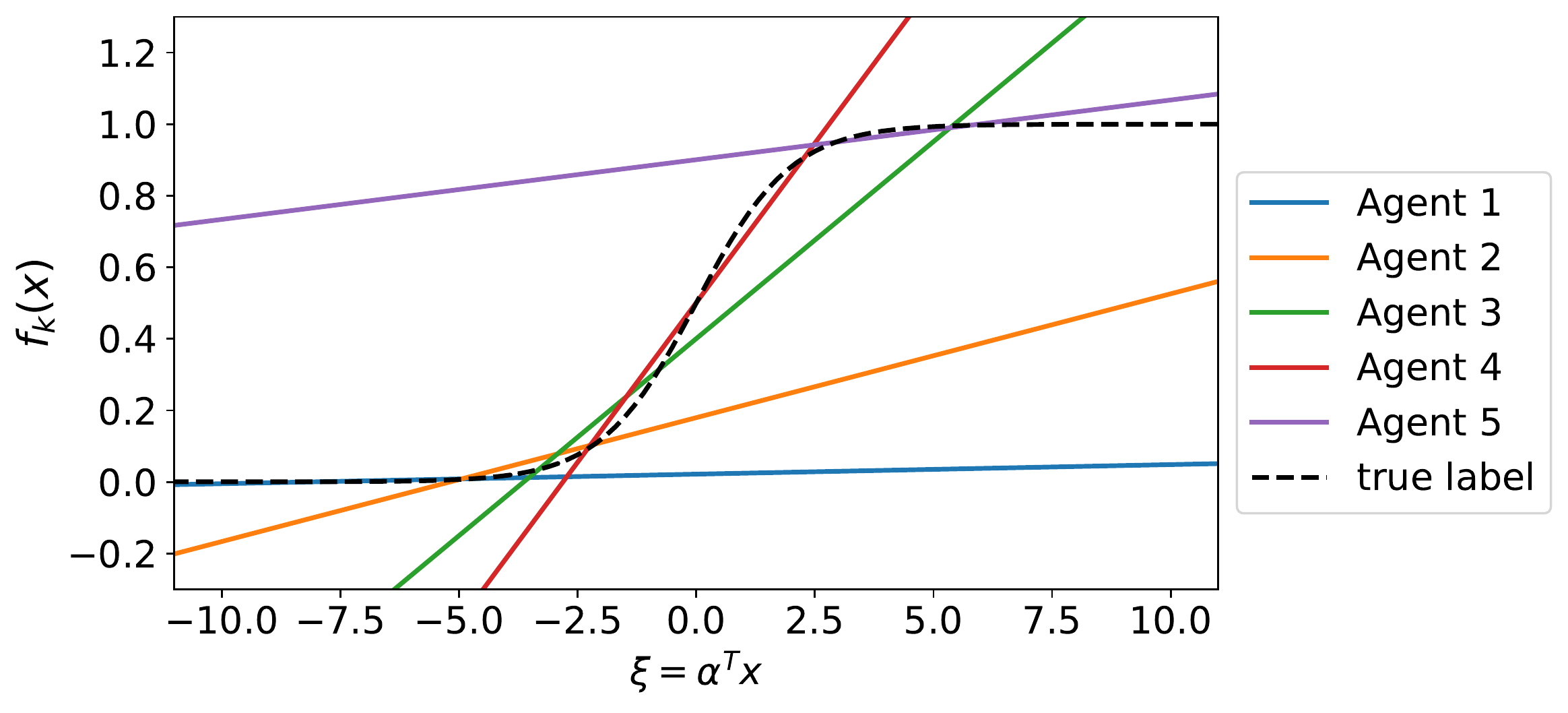}
    \caption{\emph{(agents' local models)}. Local linear models fit by individual agents for the default experimental configuration.  } 
    \label{fig:linearmodels}
\end{figure}

\subsubsection*{Additional Evaluations:}
\textit{A) Sensitivity of DeGroot to hyperparameter $N$.} In Figure~\ref{fig:nn} we demonstrate how the performance of Algorithm~\ref{alg:DG} in this synthetic setting changes with the number of neighbors $N$ used for local validation for two different levels of label noise. Overall we find that the performance of DeGroot is not very sensitive to the choice of $N$, whereas the optimal regime of values for $N$ increases with the noise in the data. In general, we recommend choosing $N$ to corresponds to approximately $1-10\%$ of the available data, with a hard lower-bound on $N$, bounding it away from $1$. For later investigations on real data we pick $N$ to be $1\%$ of the partition size.

\begin{figure}[t!]
    \centering
    \subfloat[$\sigma_Y=0.1$]{\includegraphics[width=0.48\textwidth]{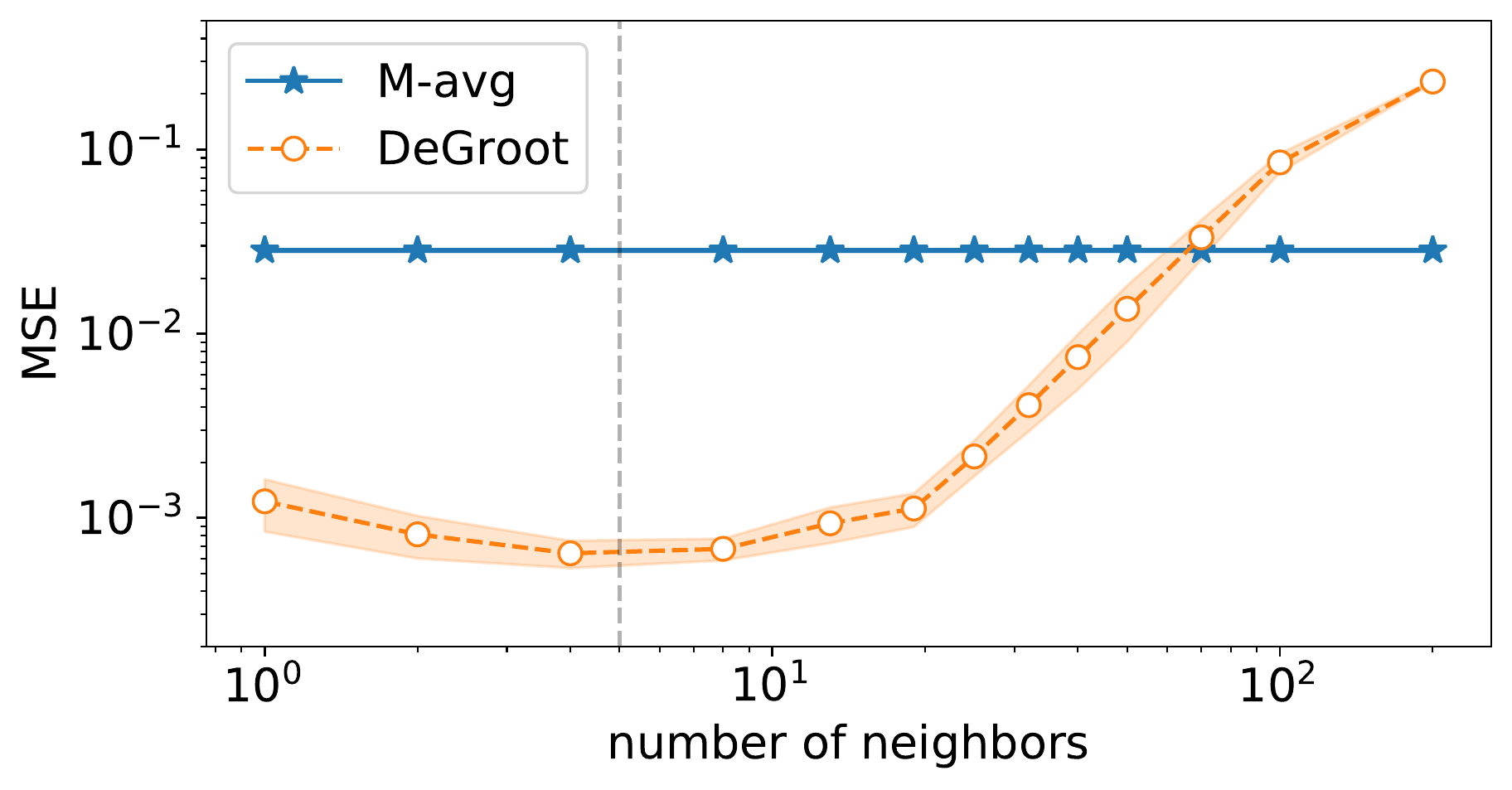}}\hfill
    \subfloat[$\sigma_Y=0.3$]{\includegraphics[width=0.48\textwidth]{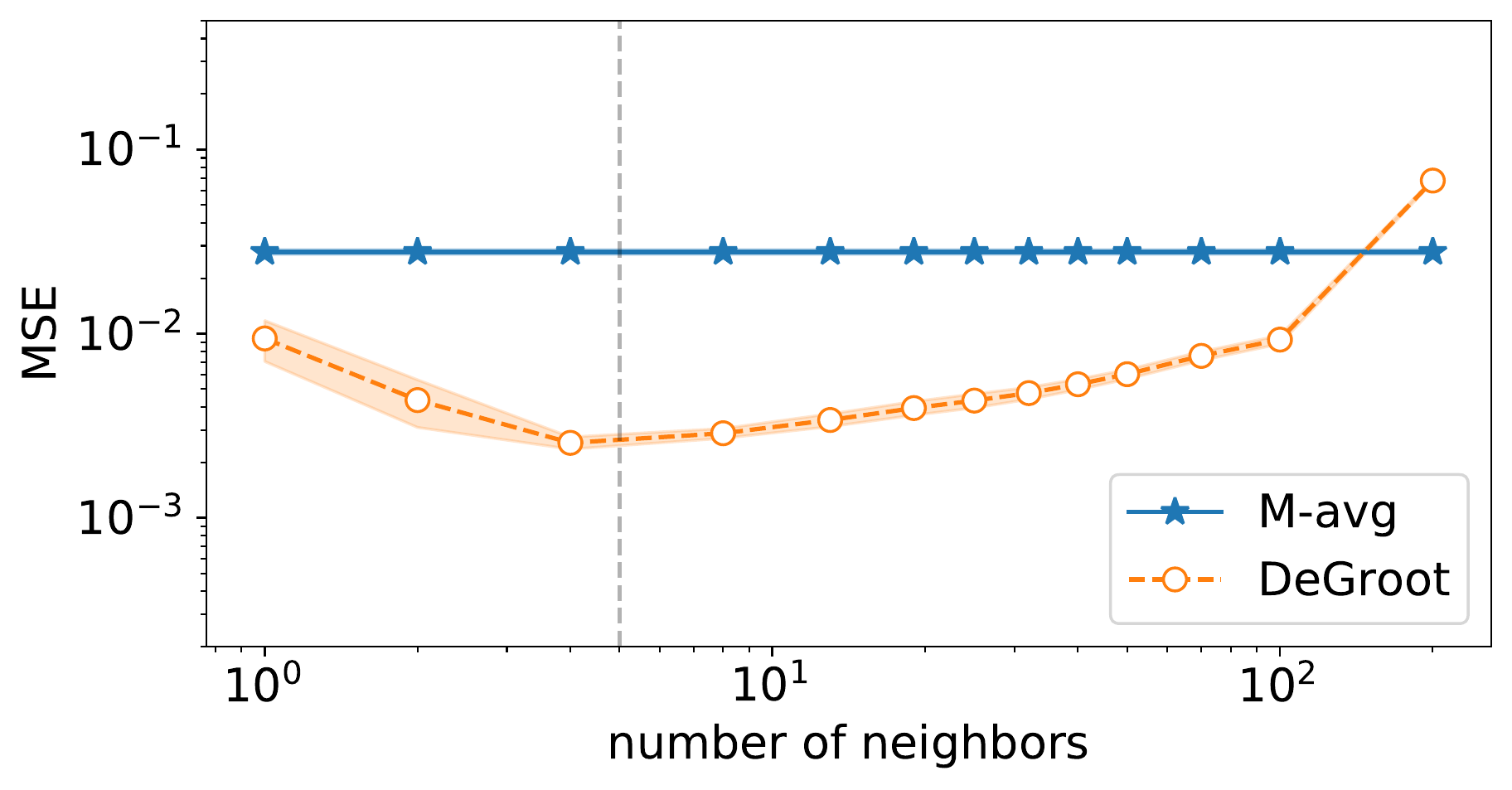}}
     \caption{\emph{(Performance vs. number of neighbors)}. Synthetic data experiment comparing DeGroot to model averaging (M-avg) for two different noise levels ($\sigma_Y$) in the training data labels. Gray dashed line indicate the default value in our experiments.}
    \label{fig:nn}
\end{figure}

\textit{B) Varying overlap in local data partitions.} In Figure~\ref{fig:cov} we investigate how the performance of DeGroot changes in comparison to M-avg for varying overlap in the local data partitions. Therefore we vary the covariance $\Sigma_k = \text{I}\sigma^ 2$ in the feature distribution of the individual agents, while keeping the means $\mu_k$ fixed. We find that for larger values of $\sigma^2$ there is less to gain for DeGroot since models are less specialized, although the gains remains significant up to $\sigma^2=10$. We further see that if the variance in the data is very small and the local datasets are not sufficiently covering the feature space, the local cross-validation procedure and the adaptivity of the DeGroot procedure start to suffer.

\begin{figure}[h!]
    \centering
\includegraphics[height=0.3\textwidth]{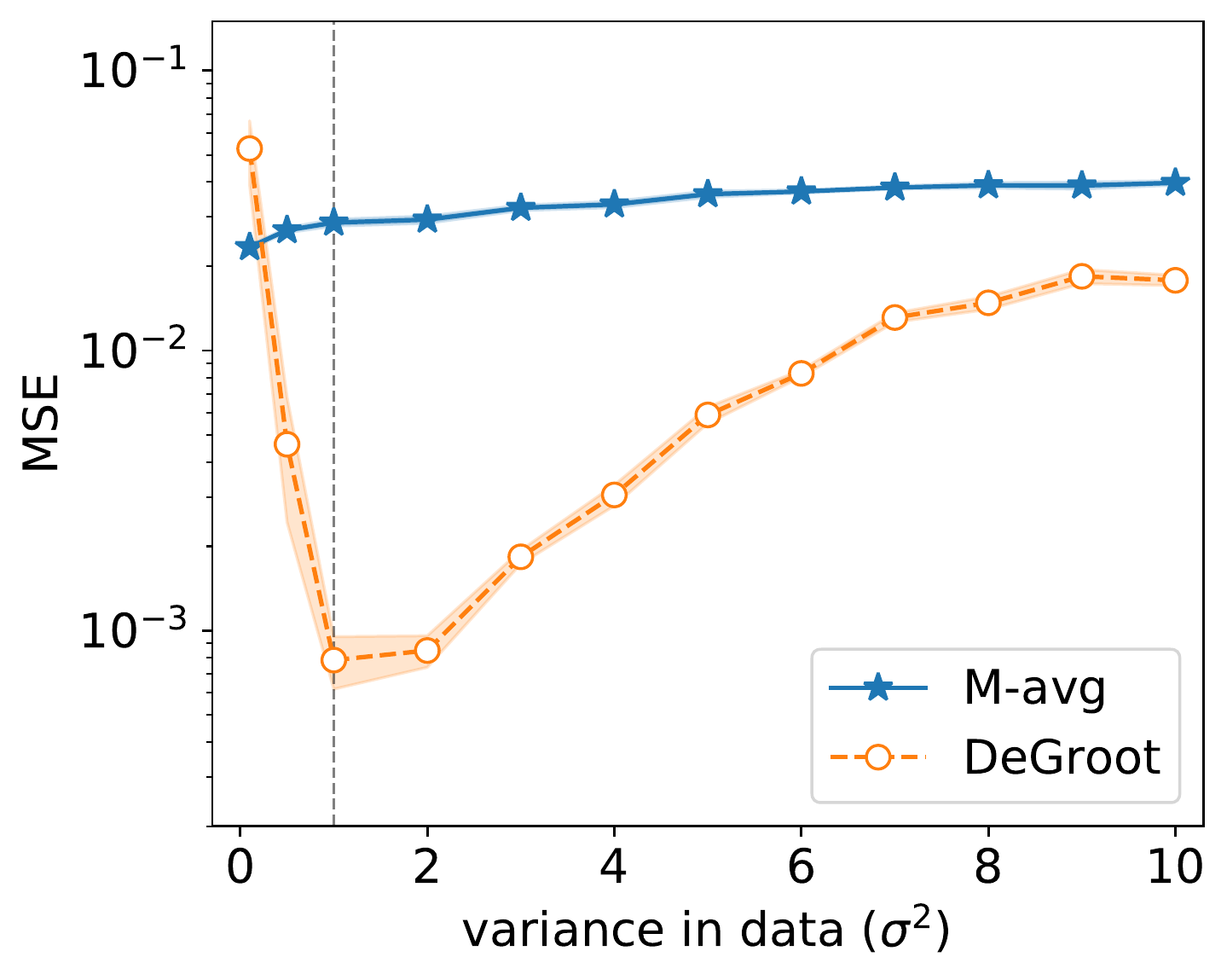}
     \caption{\emph{(Varying overlap in data partitions)}. Comparing DeGroot to classical model averaging (M-avg) for varying covariance $\Sigma_k = \text{I}\sigma^ 2$ in the feature distribution, keeping the means $\mu_k$ fixed. Gray dashed line indicate the default value in our experiments.}
    \label{fig:cov}
\end{figure}

\textit{C) Alternative pooling operations.} In Figure~\ref{fig:alternative} we aim to provide additional insights into the effect of DeGroot's iterative procedure to find consensus and aggregate the individual predictions. In Figure~\ref{fig:consensus} we have shown how the accuracy of individual agents improves with the number of pooling iterations in the DeGroot procedure. After around $20$ iterations the agents have reached consensus. In Figure~\ref{fig:alternative} we compare the accuracy of the predictions obtained through DeGroot to alternative ways of aggregating the local MSE values and trust scores into a single weight vector. These baselines are constructed for diagnostic purpose to demonstrate the value of using power iterations instead of an alternative procedure.

The first method, denoted $\tau$-avg takes the mean of the trust scores $\tau_{ij}$ across all agents $i$ to obtain the weights $w_j$ of  agent $j$ on the final prediction. As shown in Figure~\ref{fig:alternative-a} $\tau$-avg performs significantly worse than DeGroot on our synthetic example, and achieves an overall MSE of $9e^{-3}$ which is $20\times$ larger than DeGroot. This shows that the DeGroot agents find a better collective prediction through iterative pooling than they would by averaging their beliefs after only a single round of pooling.

The second method first averages the local MSE values of each agent across the different datasets as $\bar{\text{MSE}}_j = \sum_i \text{MSE}_{i,j}$ and then obtains the weights through inverse-MSE weighting as $w_j = 1 /{\bar{\text{MSE}}_j}$. Similar to $\tau$-avg, this alternative procedure performs significantly worse than DeGroot, and acheives an MSE of $1e^{-2}$.

\begin{figure}[t!]
    \centering
\subfloat[$\tau$-avg aggregation]{\includegraphics[width=0.49\textwidth]{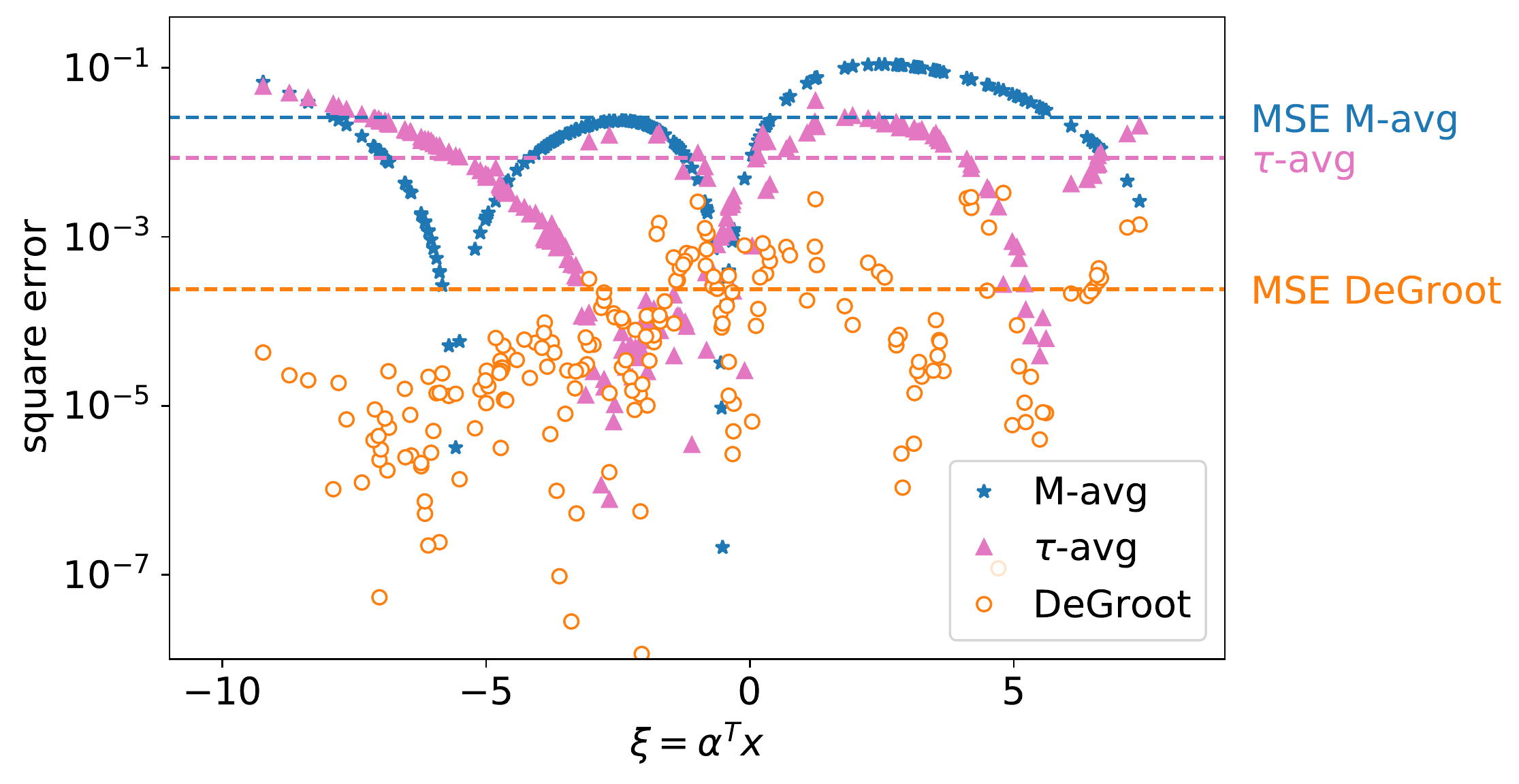} \label{fig:alternative-a}}
\subfloat[MSE-avg aggregation]{\includegraphics[width=0.49\textwidth]{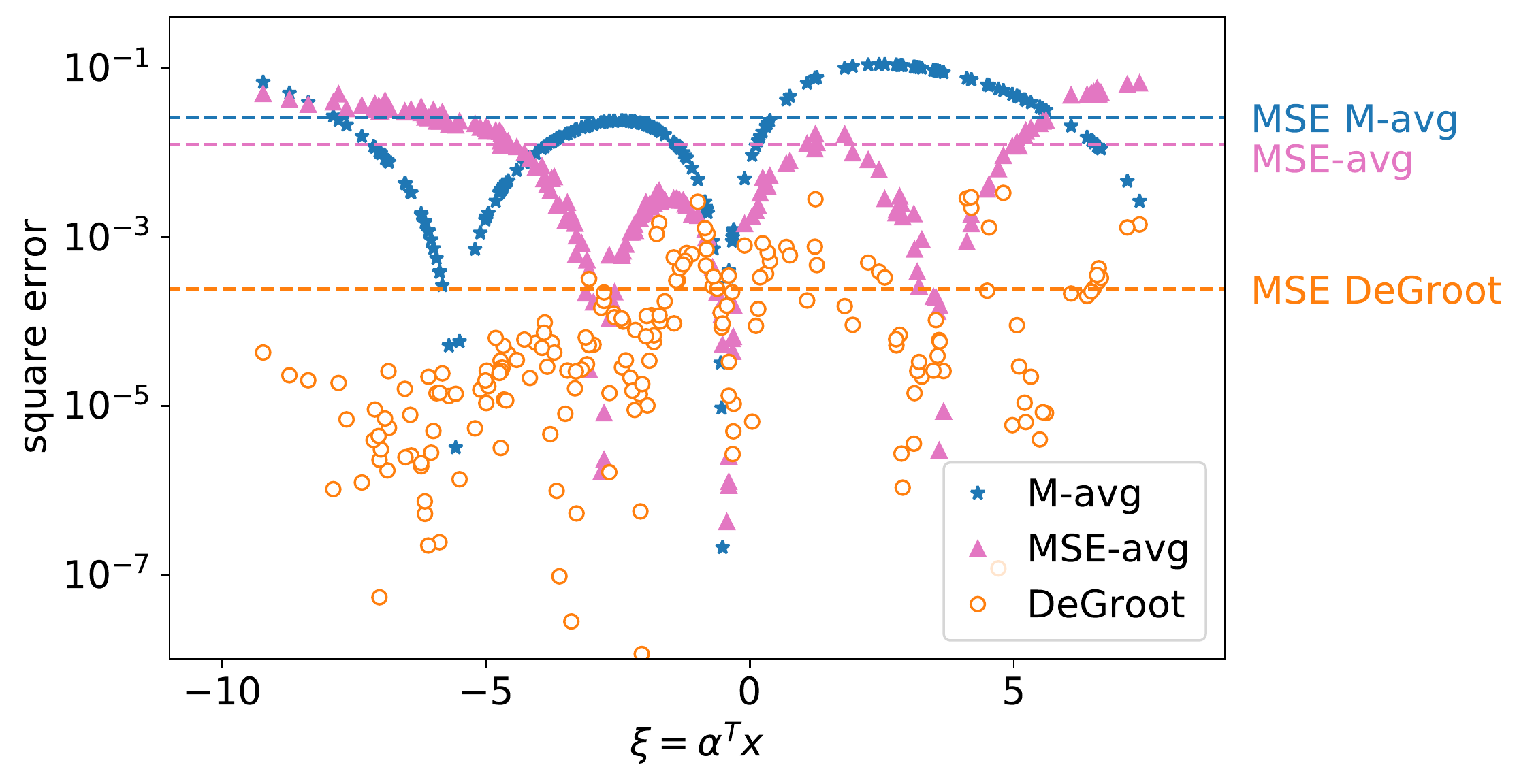}\label{fig:alternative-b}} 
\caption{\emph{(Alternative aggregation schemes)}. (a) Comparison of DeGroot to $\tau$-avg that average the trust scores across all agents to obtain the weights of the individual agents. (b) Comparison of DeGroot to MSE-avg that performs inverse-MSE weighting based on the aggregated MSE values cross the local datasets. Experiments serve for diagnosis purpose and are conducted on the synthetic setup outlined in Section~\ref{sec:synthetic}.}
    \label{fig:alternative}
\end{figure}

\textit{D) DeGroot Jackknife.} Finally, in Figure~\ref{fig:jack}, we evaluate the error bars for our predictions using the Jackknife procedure proposed in Algorithm~\ref{alg:DG_jackknife} for the test points evaluated in Figure~\ref{fig:testerror}.  Here, we report the predictions as a function of $\xi$, as well as the standard error \eqref{eq:SE}. We find that the standard errors are relatively small in the center, but increase at the edges of the space. This makes sense: on the edges of the space there is only one agent with good prediction accuracy, so the final consensus prediction is not stable to the deletion of any one agent. To validate this explanation further, in the right panel we increase the spread of the training points for each agent, so that the agents' training data overlap more. As expected, the standard errors are now much smaller---because there are multiple agents with training data in most regions, the procedure is more stable.

\begin{figure}[t!]
    \centering
\subfloat[$\sigma^2 = 1$]{\includegraphics[width=0.48\textwidth]{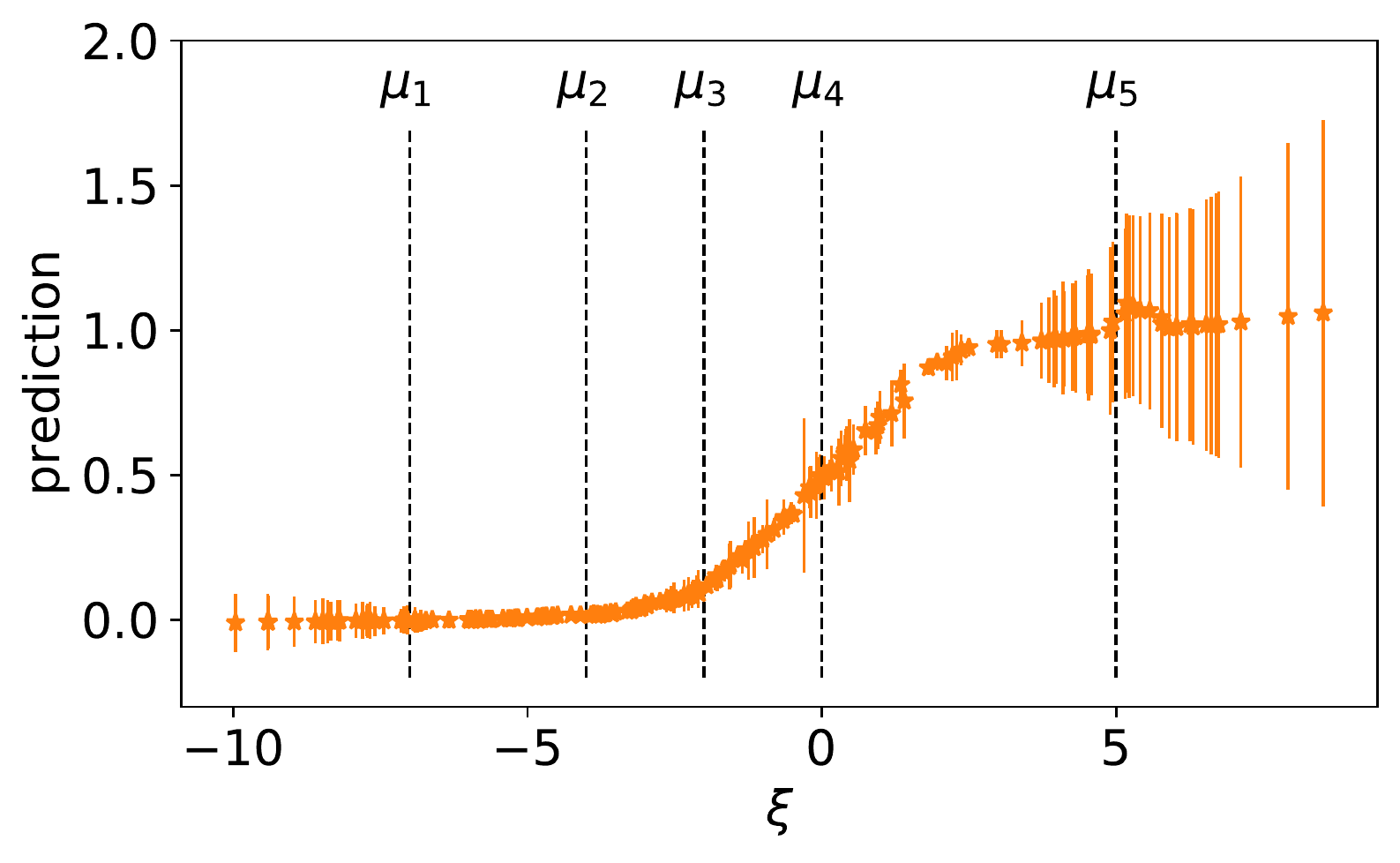} }
\subfloat[$\sigma^2 = 5$]{\includegraphics[width=0.48\textwidth]{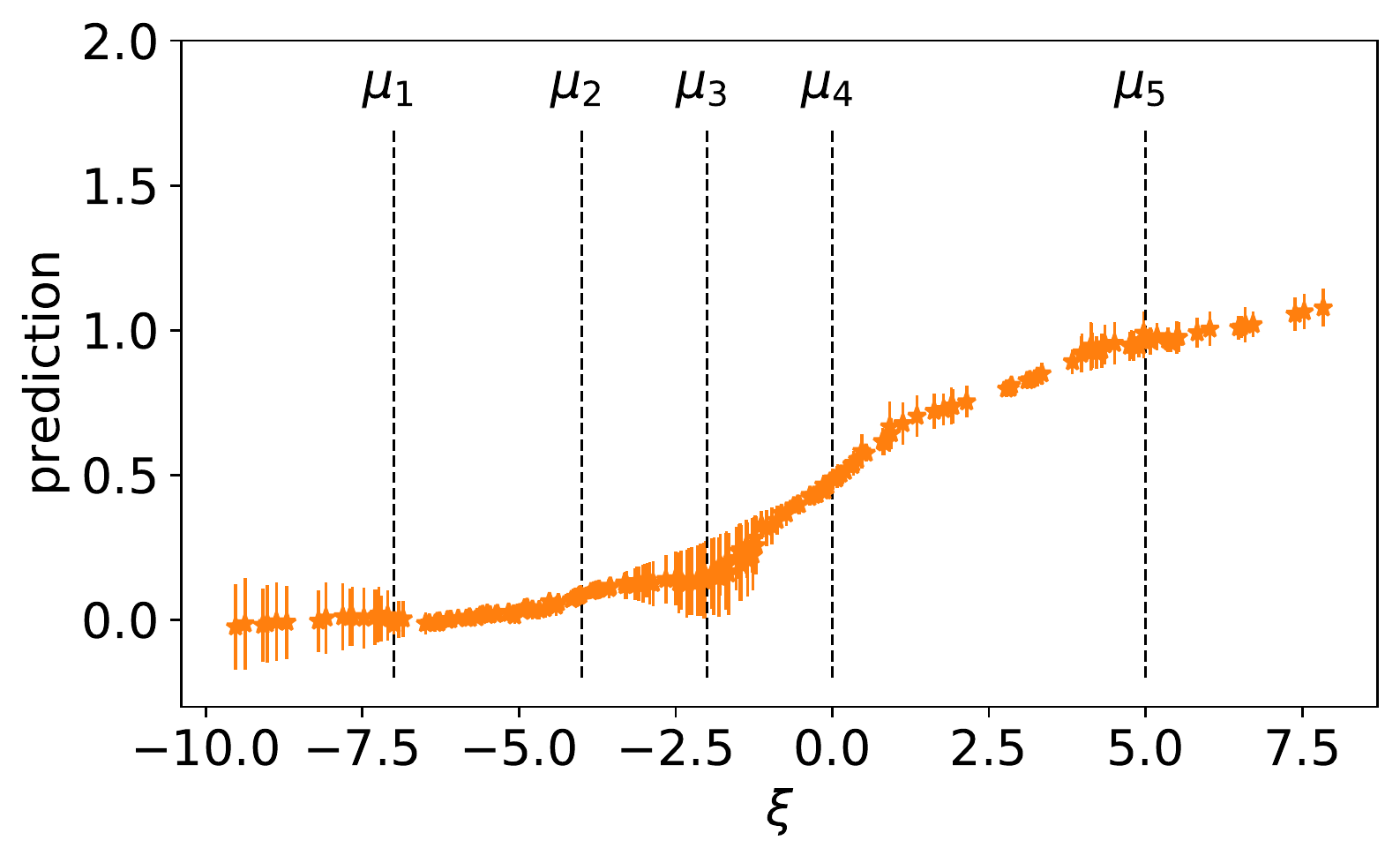} }
\caption{\emph{(DeGroot Jackknife)}. Evaluations of error bars using the decentralized Jackknife procedure proposed in Algorithm~\ref{alg:DG_jackknife}. The test points correspond to the same 200 test points visualized in Figure~\ref{fig:testerror}. For (a) we use the standard configuration, and for (b) we increase the covariance of the local training datasets by a factor of $5$, i.e. $\Sigma_k = 5 \text{ I}$.}
    \label{fig:jack}
\end{figure}

\subsection{Evaluation on Real Data}
\label{app:realdata}

\textit{Datasets.} The specifications of the datasets used for our experiments in Section~\ref{sec:realdata} can be found in Table~\ref{tab:datasets}. The datasets have been downloaded from libsvm~\citep{libsvm} and are used without any additional preprocessing, apart from the boston dataset that is normalized for neural network training. The boston housing price data~\citep{Dua:2019} and the E2006 datasets~\citep{E2006dataset} have continuous outcome variables. Abalone~\citep{Dua:2019} has integer values outcomes from $1-29$, YearPrediction~\citep{Dua:2019} has integer valued year numbers from $1922-2011$, and the labels of cpusmall correspond to integers from $0-99$. 

\begin{table}[h!]
    \centering
    \vspace{-0.1cm}
    \caption{Regression datasets used for evaluation, downloaded from \citep{libsvm}.\vspace{0.3cm}}
    \begin{tabular}{l l c c}
    \toprule
                 & &\# samples & \# featues \\
                 \midrule
         &Boston &506&13\\ 
         &E2006 &16087& 3308\\ 
         &Abalone &4077&8  \\
         &cpusmall &8192& 12\\ 
         &YearPrediction &463715& 90\\
         \bottomrule
    \end{tabular}
    \label{tab:datasets}
\end{table}

\subsubsection*{Details on Benchmark Study in Table~\ref{tab:benchmark}:} 
\textit{Baselines.} We compare DeGroot to three baselines: M-avg, CV-static, CV-dynamic. The three methods describe different approaches to determine the weight $w_j$ of each model on the final prediction $p(x')= \sum_j w_j f_j(x')$. 
(M-avg) corresponds to the most natural baseline of equally weighted averaging with $w_j = \frac 1 K$ 
for all $j\in[K]$. This baseline that is optimal if models are unbiased and have equal variance. 
(CV-static) and (CV-adaptive) are two methods that have access to additional hold-out data to evaluate the predictive accuracy $\text{MSE}_j$ of each model $j\in[K]$ and determine the weights $w_j$. As the name says (CV-static) uses static weights, and (CV-adaptive) uses an adaptive weighting scheme. Similar to DeGroot both methods use weights that are inversely proportional to the MSE of the models:
\[w_j = \frac{1/\text{MSE}_j}{\sum_j 1/\text{MSE}_j}\]
In CV-static $\text{MSE}_j$ is evaluated on the hold-out data, and in CV-adaptive $\text{MSE}_j$ is evaluated on a subset of the hold-out data, composing of the $N$ closest points to $x'$. We use the same distance measure, and number of neighbors $N$ as we use to perform cross-validation on the individual agents in DeGroot. Our goal is to minimize potential confounding for a most meaningful comparative investigation.

\textit{Evaluation.}
In all experiments the data is first randomly partitioned into train,test and validation data.  The validation data is completely ignored by DeGroot and M-avg and only used for CV-static and CV-adaptive. The size of the hold-out data is equal to the size of one data partition. We use $n_{test}=\max(0.15 n, 500)$ test samples. 
For every random split all methods are evaluated on the same set of local models.
We tune the hyparparameter of each model using a rough hyperparameter search upfront, but our main goal is not to get optimal accuracy with each individual model, but rather to investigate how the different aggregation procedures can deal with (resonable) given models. Hyperparameters are reported in Table~\ref{tab:benchmark}. 

\textit{Models.}
For training of the models we use the built in model classes for scikit-learn~\citep{scikit-learn}. Ridge and Lasso are imported from \texttt{sklearn.linear\_model} and the hyperparameter $\lambda$ corresponds to the regularizer strength. For the DTR we use the \texttt{DecisionTreeRegressor} from \texttt{sklearn.tree} where the hyperparameter max\_depth corresponds to the maximum depth of the regression tree. For the neural network (NN) we use the \texttt{MLPRegressor} class from \texttt{sklearn.neural\_network}, it uses relu activation by default, and we choose two hidden layers with 8 neurons each, and $\alpha$ corresponds to the regularization strength. All other parameters are set to their default values. Whenever a NN model has not converged within the given number of iterations (500 for E2006 and 200 for boston) the model is nevertheless included in the ensemble since the purpose of this study is the evaluate how the different schemes deal with given models that might be heterogeneous in terms of quality.

\textit{Confidence Intervals.}
The reported confidence intervals are over the randomness in the data splitting and partitioning, they show one standard deviation from the mean.

\subsubsection*{Additional Evaluations:}

\textit{Heterogeneity.} 
We perform additional evaluations with the different approaches of achieving heterogeneity described in Section \ref{sec:realdata}. In Figure~\ref{fig:heter-a1} and Figure~\ref{fig:heter-a2} we show results for ridge regression on the abalone, boston and cpusmall dataset.
The gain of DeGroot over M-avg consistently grows with the degree of heterogeneity. 

\begin{figure}[h!]
    \centering
\subfloat[Abalone ($\lambda'=1$)]{\includegraphics[width=0.33\textwidth]{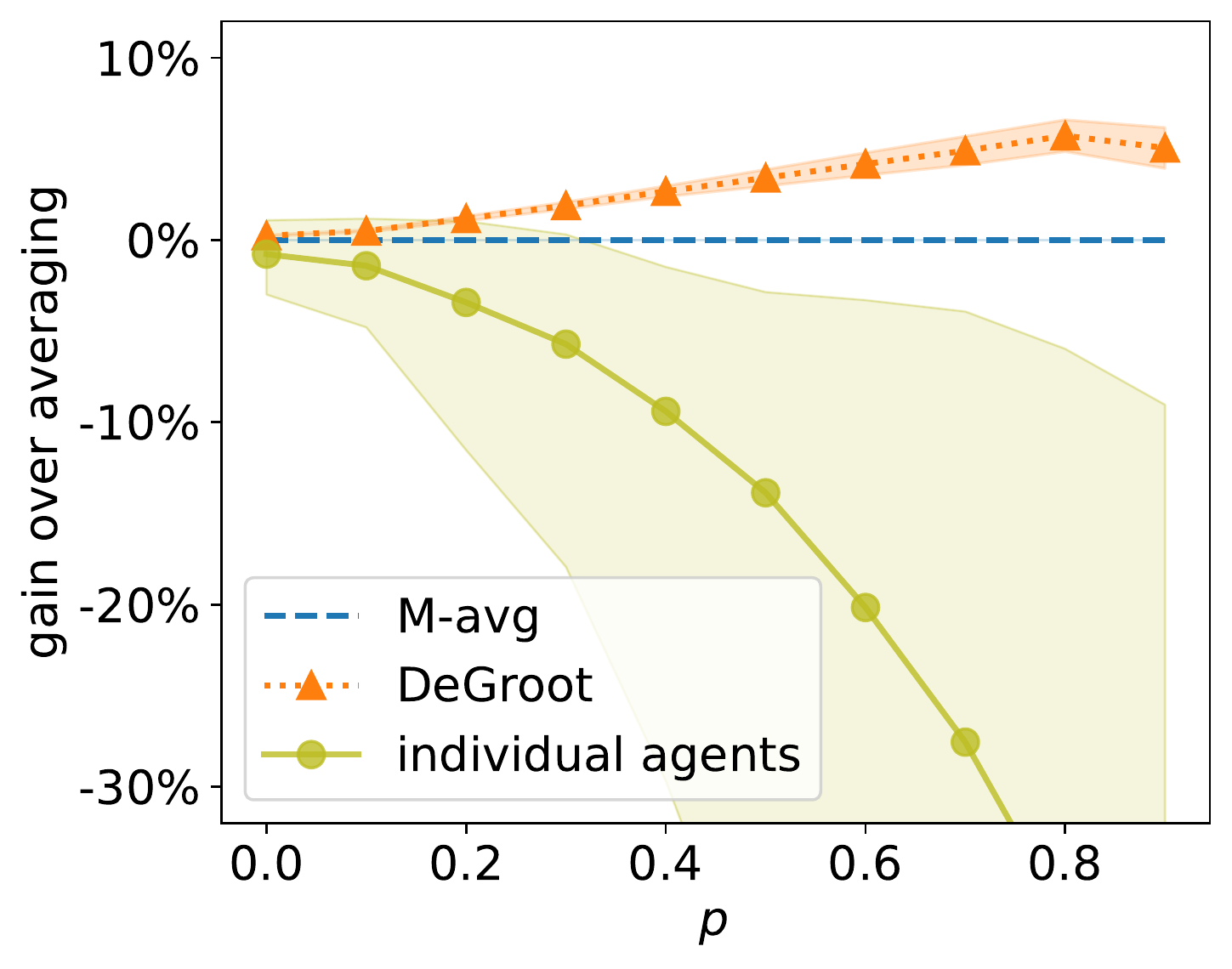}
    }
    \subfloat[Boston ($\lambda'=10^{-5}$)]{\includegraphics[width=0.33\textwidth]{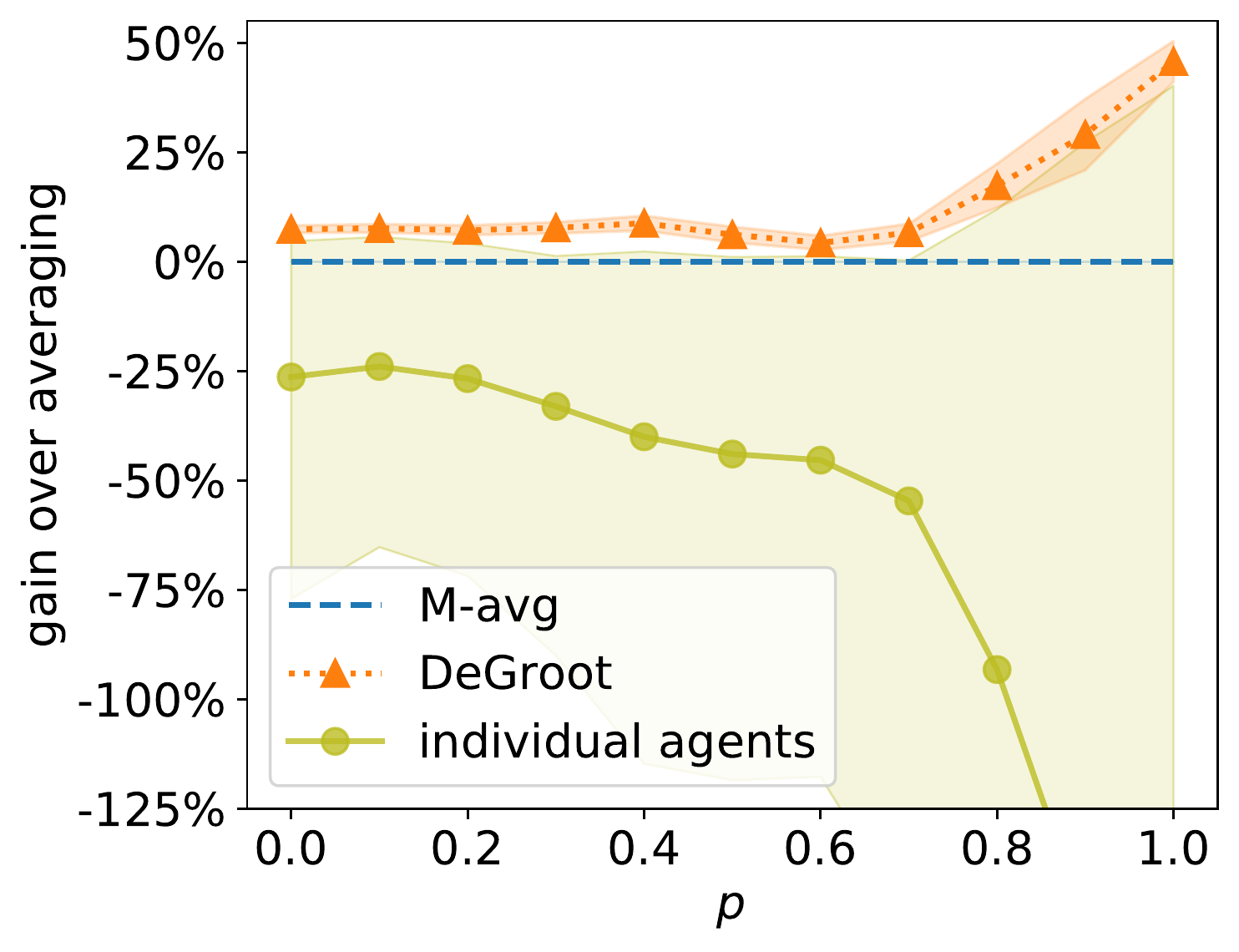}
    }
    \subfloat[Cpu-small ($\lambda'=1e^{-5}$)]{\includegraphics[width=0.33\textwidth]{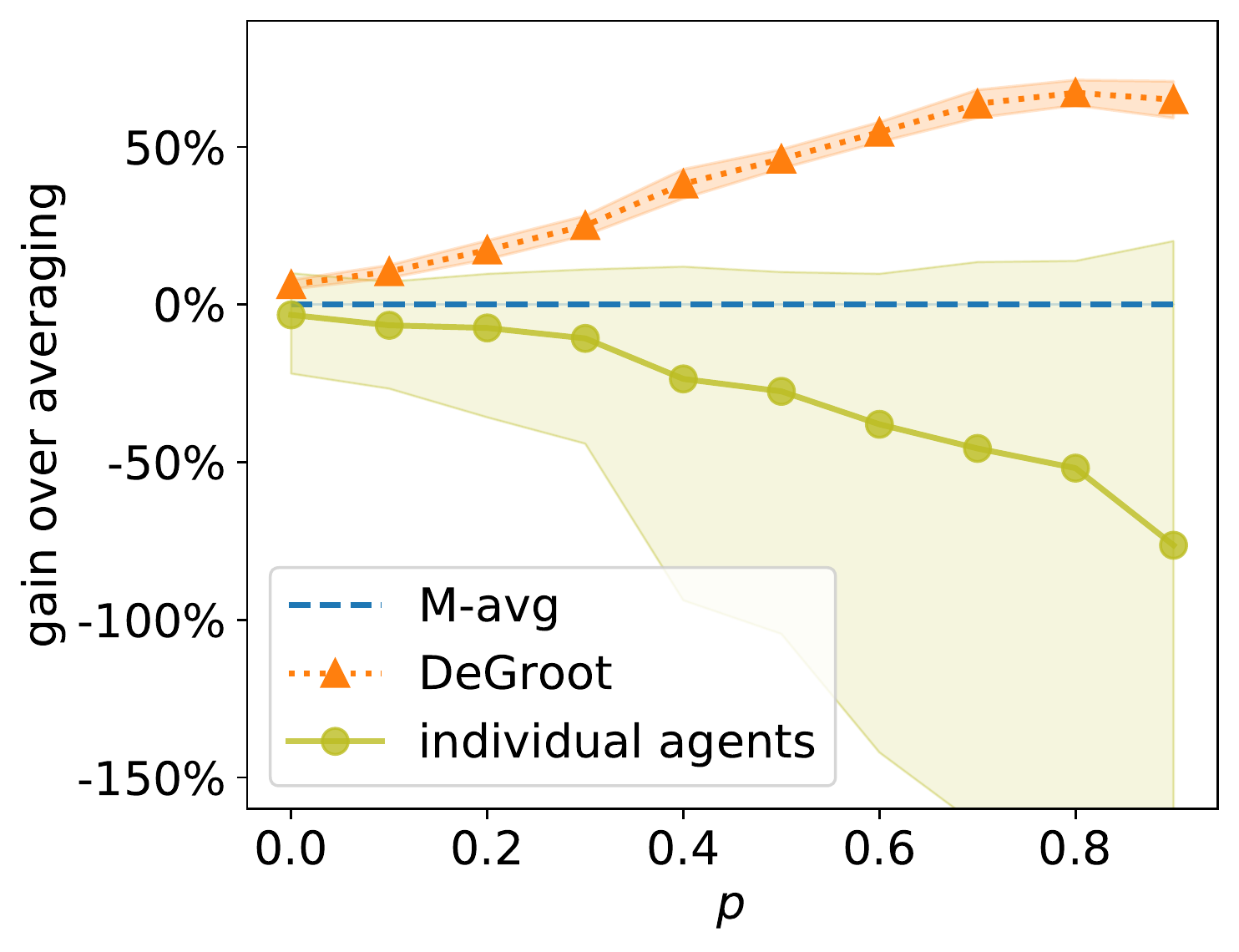}
    }
     \caption{\emph{(Label heterogeneity)}. Relative gain of DeGroot over M-avg. Same setting as Figure~\ref{fig:heter1} for ridge regression on different datasets. We depict the performance of the individual models in the ensemble using the shaded green area, the upper edge corresponds to the best and the lower edge to the worst performing model, the green line indicating the average performance across the models. \vspace{-0.2cm}}
    \label{fig:heter-a1}
\end{figure}

\begin{figure}[h!]
    \centering
\subfloat[Abalone ($\lambda'=1$)]{\includegraphics[width=0.33\textwidth]{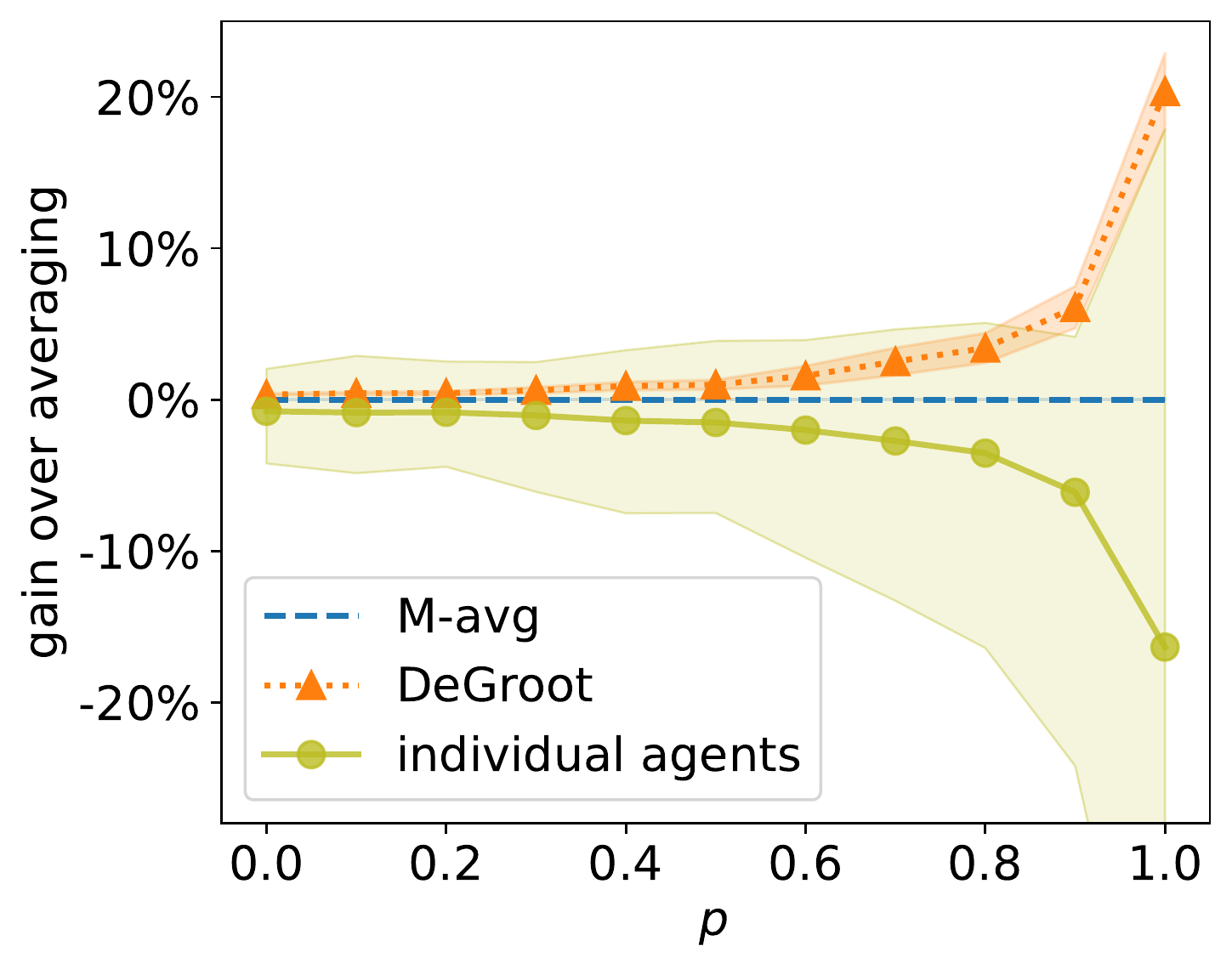}
    }
    \subfloat[Boston ($\lambda'=10^{-5}$)]{\includegraphics[width=0.33\textwidth]{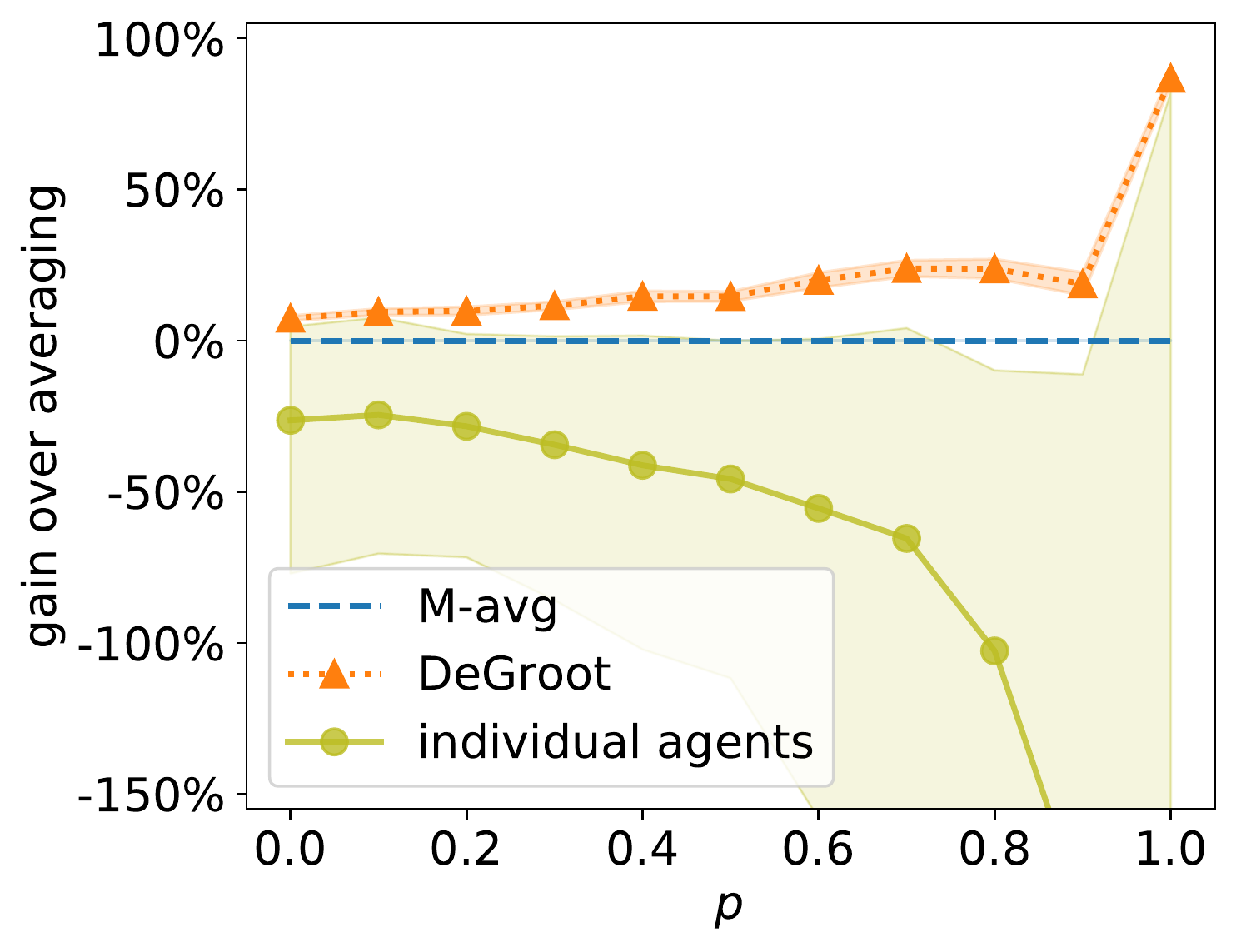}
    }
    \subfloat[Cpu-small ($\lambda'=1e^{-5}$)]{\includegraphics[width=0.33\textwidth]{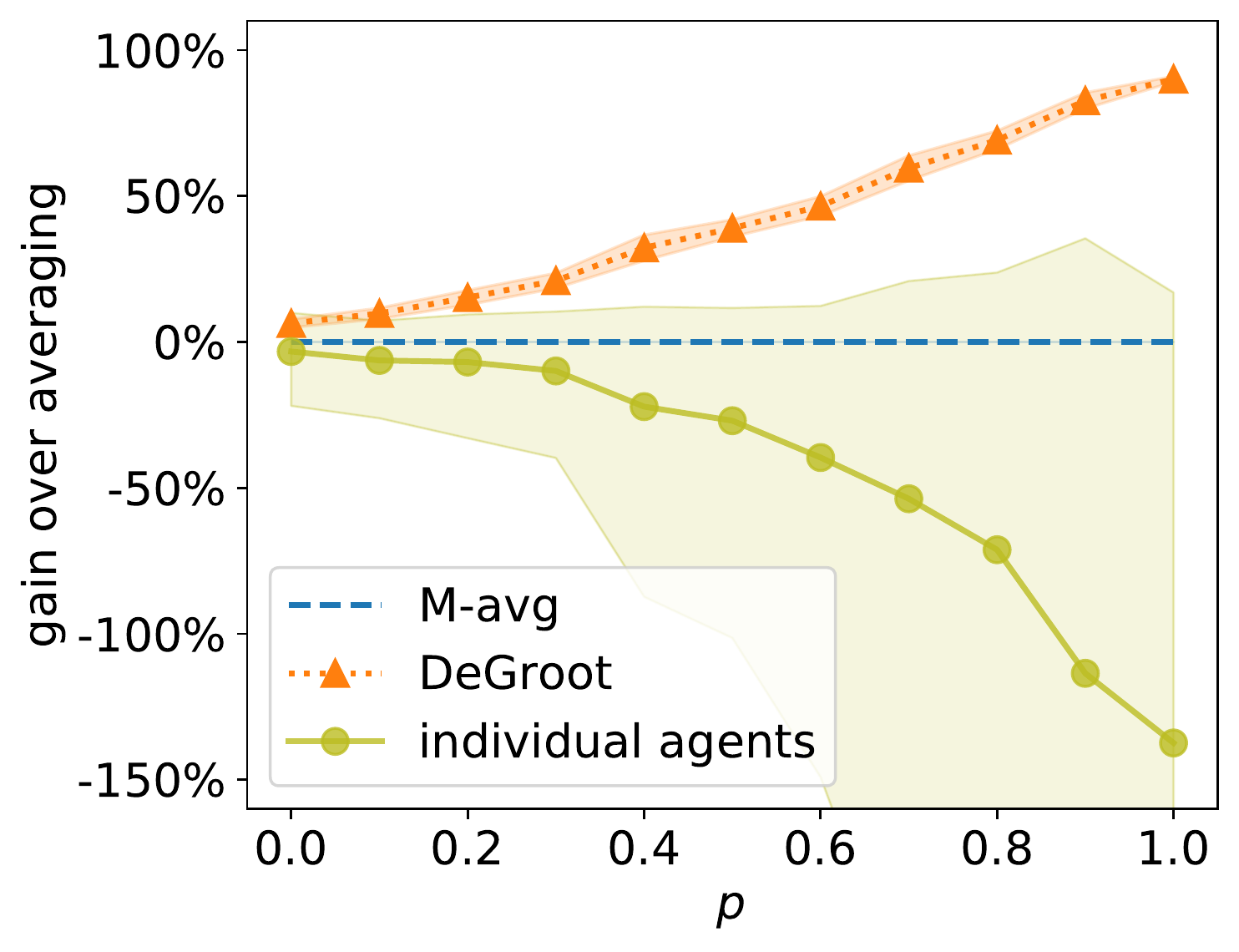}
    }
     \caption{\emph{(Feature heterogeneity)}. Relative gain of DeGroot over M-avg. Same setting as Figure~\ref{fig:heter3} for ridge regression on different datasets. We partially sort the training data along feature 8 (Abalone) feature 9 (Boston) and feature 11 (cpusmall). The performance of the individual models in the ensemble is depicted using the shaded green area, the upper edge corresponds to the best and the lower edge to the worst model, the green line indicating the average performance across the models. \vspace{-0.2cm}}
    \label{fig:heter-a2}
\end{figure}

\emph{Scalability.}
In Figure~\ref{fig:scaling} we investigate the scaling of the DeGroot method on two different datasets. We show results for 1 up to 128 agents. The training data (after test set hold-out) is partitioned across the agents by partially sorting the label with $p=0.5$. This means with increasing number of agents $K$ the partitions get smaller. We choose the number of neighbors in DeGroot as $N=\max(2,0.01*{n}_{\text{local}})$, where ${n}_{\text{local}}$ denotes the size of a local partition.
We find that DeGroot scales robustly with the number of agents in the ensemble and outperforms averaging by a large margin. Remarkably, the adaptive weighting of DeGroot is still effective for 128 agents on the abalone datasets, where each agent only has access to $30$ samples. 

\begin{figure}[h!]
    \centering
    \subfloat[Cpusmall -- ridge ($\lambda'=1e^{-5}$)]{\includegraphics[width=0.33\textwidth]{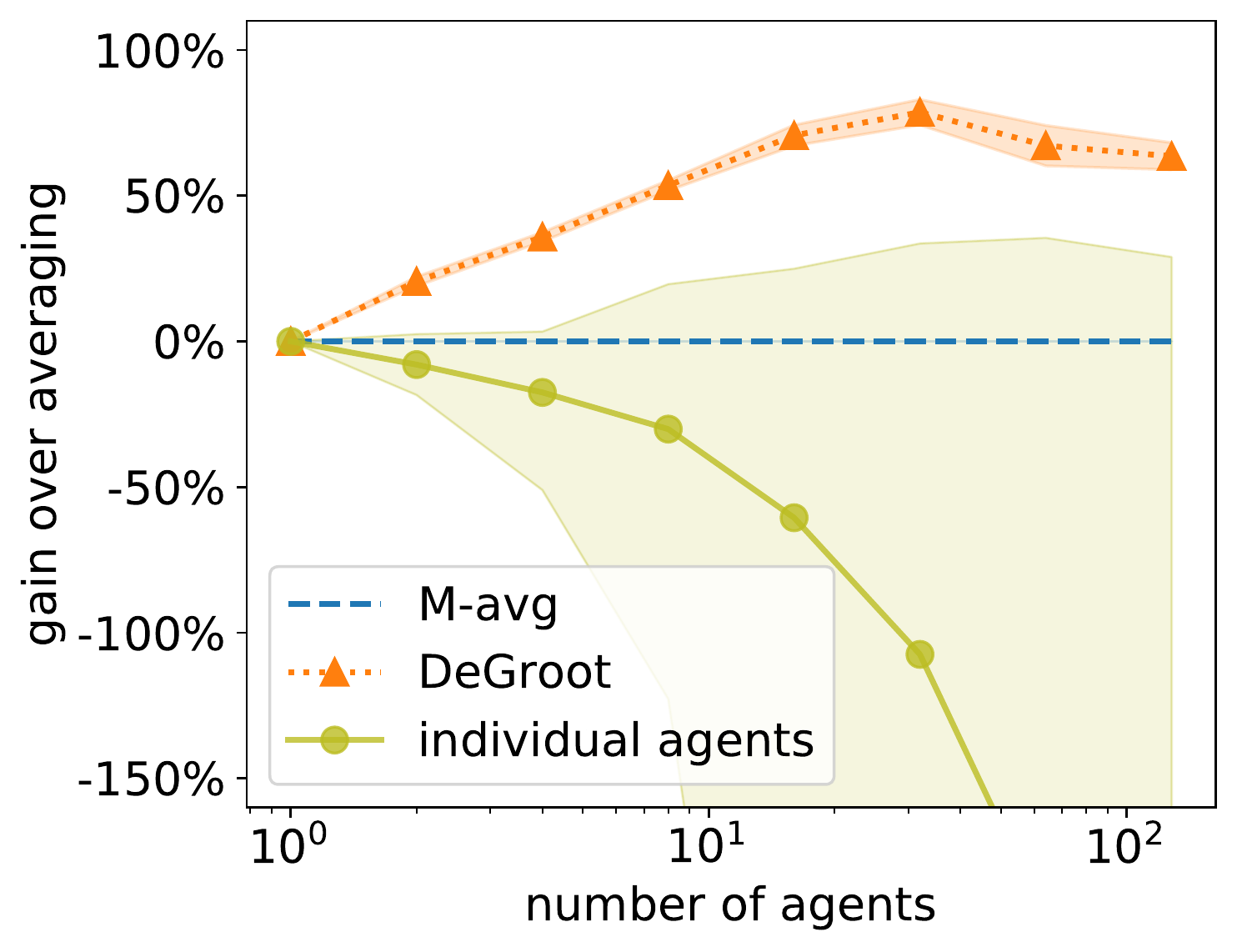}}
    \hspace{2cm}
    \subfloat[Abalone -- lasso ($\lambda'=1e^{-5}$)]{\includegraphics[width=0.33\textwidth]{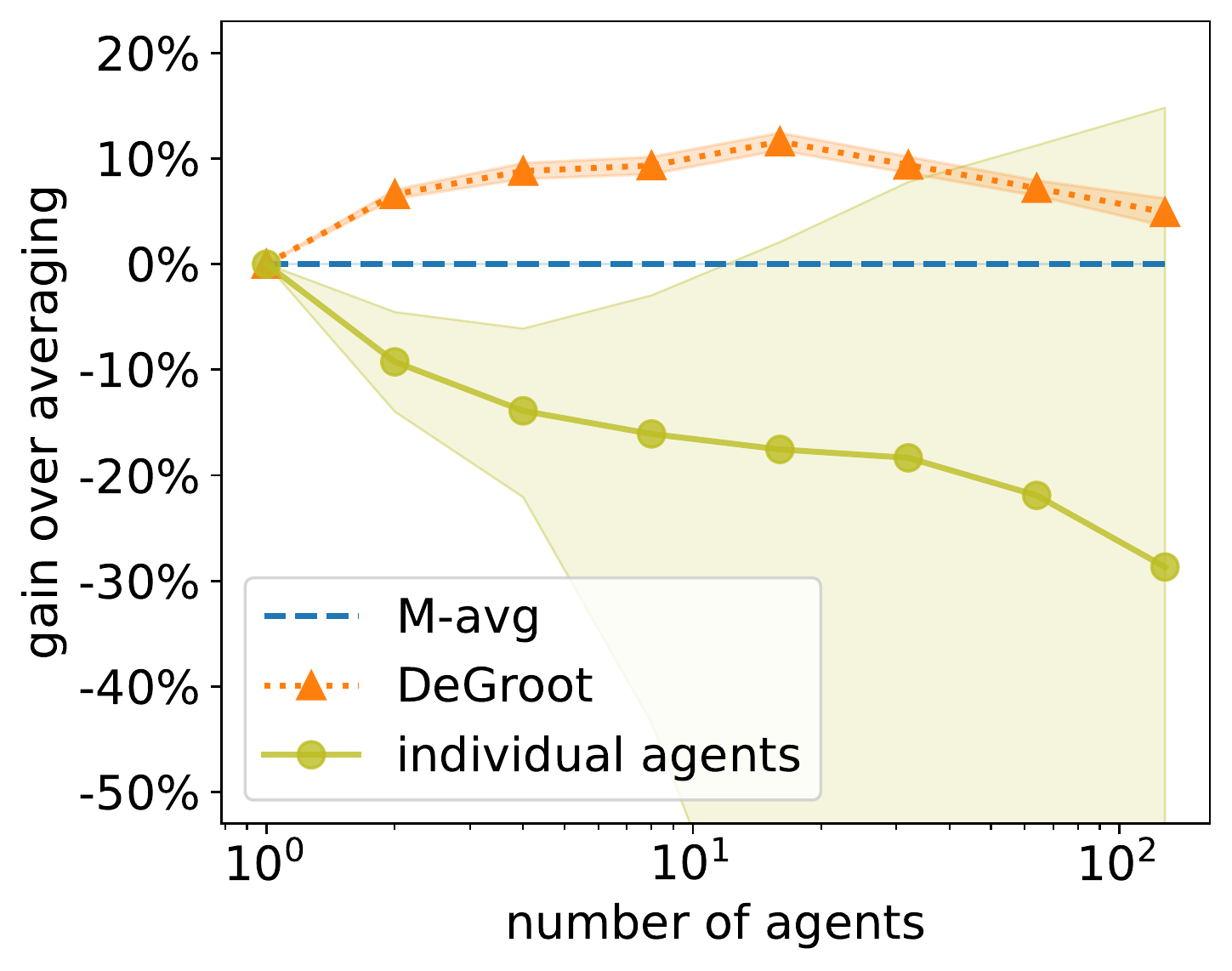}}
     \caption{\emph{(Scaling behavior of DeGroot)}. Relative gain of DeGroot aggregation over model averaging (M-avg) for different number of agents. Experiments performed for two different configurations. }
      \label{fig:scaling}
\end{figure}

\end{document}